\documentclass{l4dc2025}

\usepackage{microtype}
\usepackage{amsfonts}
\usepackage{booktabs} % for professional tables
\usepackage{times}
\usepackage{xcolor}  % Add this to your preamble
\usepackage{algorithm}
\usepackage{algorithmic}
\usepackage{makecell} 
\usepackage{multirow}  % For merging rows
\usepackage{array} 
\usepackage{siunitx}

\usepackage{graphicx} % Required for inserting images
\usepackage{subcaption}% Offers an user interface to sub-captions.
\usepackage{float}% Improves the interface for defining floating objects

\definecolor{lightblue}{RGB}{51, 153, 255}
\definecolor{lightred}{RGB}{255, 51, 51}

\title[Inverse Delayed RL]{Inverse Delayed Reinforcement Learning}

\author{
 \Name{Simon Sinong Zhan$^{*1}$} \Email{SinongZhan2028@u.northwestern.edu}\\
 \Name{Qingyuan Wu$^{*2}$} \Email{qingyuan.wu@soton.ac.uk}\\ 
 \Name{Zhian Ruan$^1$} \Email{zhianruan2025@u.northwestern.edu}\\
 \Name{Frank Yang$^1$} \Email{frankyang2024@u.northwestern.edu}\\
 \Name{Philip Wang$^1$} \Email{philipwang2025@u.northwestern.edu}\\
 \Name{Yixuan Wang$^1$} \Email{yixuanwang2024@u.northwestern.edu}\\
 \Name{Ruochen Jiao$^1$} \Email{RuochenJiao2024@u.northwestern.edu}\\
 \Name{Chao Huang$^2$} \Email{Chao.Huang@soton.ac.uk}\\ 
 \Name{Qi Zhu$^1$} \Email{qzhu@northwestern.edu}\\
 \addr $^1$  Department of Electrical and Computer Engineering, Northwestern University, USA\\
 \addr $^2$ School of Electronics and Computer Science, University of Southampton, UK\\
}

\begin{document}
\maketitle
\def\thefootnote{*}\footnotetext{These authors contributed equally to this work}
%%%%%%%%%%%%%%%%%%%%%%%%%%%%%%%%%%%%%%%%%%%%%%%%%%%%%%%%%%%%%%%%%%%%%%%%%%%%%%%

\begin{abstract}
Inverse Reinforcement Learning (IRL) has demonstrated effectiveness in a variety of imitation tasks. In this paper, we introduce an IRL framework designed to extract rewarding features from expert trajectories affected by delayed disturbances. Instead of relying on direct observations, our approach employs an efficient off-policy adversarial training framework to derive expert features and recover optimal policies from augmented delayed observations. Empirical evaluations in the MuJoCo environment under diverse delay settings validate the effectiveness of our method. Furthermore, we provide a theoretical analysis showing that recovering expert policies from augmented delayed observations outperforms using direct delayed observations.
\end{abstract}


%%%%%%%%%%%%%%%%%%%%%%%%%%%%%%%%%%%%%%%%%%%%%%%%%%%%%%%%%%%%%%%%%%%%%%%%%%%%%%%
\section{Introduction}
Reinforcement Learning (RL) has achieved remarkable success across diverse domains, including video and board games \citep{berner2019dota,silver2017mastering}, robotics \citep{kormushev2013reinforcement}, and safety-critical autonomous systems \citep{wang2023joint,wang2023enforcing,zhan2024state}. Despite these advancements, RL heavily relies on the quality of the reward function, which often demands significant domain expertise, labor, and time to design~\citep{russell1998learning}. To address this challenge, \cite{kalman1964linear} introduced the concept of the inverse problem of optimal control theory, providing a way to bypass explicit reward or cost function specification. With the integration of machine learning techniques, Imitation Learning (IL) has evolved into two main branches. Behavior Cloning (BC~\citep{torabi2018generative}) directly learns from expert demonstrations by aligning with the distribution of expert behaviors. In contrast, Inverse Reinforcement Learning (IRL~\citep{arora2021survey}) focuses on extracting reward functions from expert behavior to guide policy learning. These methods have shown significant promise in real-world applications, including autonomous driving~\citep{codevilla2018end} and legged locomotion~\citep{peng2020learning}. However, expanding IL to broader applications still requires attention to address real-world environmental challenges, such as hybrid dynamic~\citep{hempel2014inverse,su2024switching}, stochasticity~\citep{zhan2024model}, and delays~\citep{xue2020over}. In this paper, we consider IRL under the delay setting. Specifically, how can we extract rewarding features from expert trajectories when the provided trajectories are interfered with delays?

Delay disturbance is prevalent in nowadays timing-sensitive Cyber-Physical Systems applications~\citep{xue2017safe,feng2019taming,liu2023safety} including robotics~\citep{quadrotor_delay, arm_delay}, transportation systems~\citep{cao2020using}, and financial trading~\citep{hasbrouck2013low}, where feedback delays are inherent in actuation, transmission, and sensors processing~\citep{sun2021relectrode,sun2022microfluid}. Under the RL formulation, delays can be primarily categorized into three sections, namely observation delay, action delay, and reward delay~\citep{liotet2022delayed}. Unlike well-explored reward delay~\citep{han2022off,zhang2023multi}, observation delay and action delay disrupt the Markovian property of system dynamic, posing significant challenges to both RL and IRL methods. There are several methods to tackle the delayed RL problem. Augmentation-based approaches~\citep{altman_delay,katsikopoulos2003markov} reformulate the delayed RL problem back to the MDP by stacking the latest observed state with a sequence of actions that happened within the delay period transforming them into a new observation information state. Imitation-based approaches~\citep{liotet2022delayed,xie2023addressing,wu2024variational} formulate delayed RL into an imitation problem to the delay-free scenario through constructing explicit delay belief functions.

Despite the advancements in delayed RL, limited attention has been given to addressing delay in IRL. While some studies explore reward delays within the context of imitation learning~\citep{krishnan2016hirl,krishnan2019swirl}, few—if any—consider observation or action delays, where the expert trajectories consist solely of delayed observations and corresponding action sequences. Addressing delay in IRL has significant implications across various applications. For instance, in teleoperated robotics, such as remote robotic surgery or underwater exploration, automating these tasks using imitation learning requires adapting to signal delays not only during deployment but also within the expert demonstrations themselves. The difficulties are the following: The temporal misalignment between the perceived state and corresponding actions renders direct behavior cloning infeasible. Moreover, this misalignment, coupled with the non-Markovian nature of delayed observations, exacerbates the inherent ill-posedness of IRL. This ill-posedness arises from the fact that multiple reward functions can potentially explain the expert behavior observed in the trajectories. Thus, the reward functions extracted from the delayed observation state might not be optimal for recovering the expert performance. 

In this paper, we propose an off-policy framework for Inverse Delayed Reinforcement Learning (IDRL). By applying a state augmentation technique to both offline expert trajectories and online environmental interactions, our approach constructs a richer feature space that incorporates augmented state and action information. Using this augmented representation, we extract reward features that better capture the underlying dynamics. These reward signals are then integrated with state-of-the-art policy optimization methods, enabling our framework to achieve superior performance in recovering expert behavior compared to naive IRL baselines built on delayed observation states. Furthermore, we provide a theoretical analysis that justifies the use of augmented state representations over delayed observation states in the context of IRL, offering deeper insights into their effectiveness.

We first introduce related literature on Delayed RL and IRL in \sectionref{sec::related_work}, following a brief background introduction for problem formulation (\sectionref{sec::prelim}).
In \sectionref{sec::theoretical_analysis}, we present a comprehensive theoretical analysis of the performance difference using different reward shaping.
Inspired by the theoretical results, in \sectionref{sec::inverse_drl}, we specifically propose IDRL, an off-policy Inverse Delayed RL method.
In \sectionref{sec::experiment}, we show that our IDRL significantly outperforms various baselines over different MuJoCo benchmarks.

\section{Related Works}
\label{sec::related_work}
\paragraph{Delayed RL.}
Delayed signals within the RL setting can be categorized into three scenarios namely rewards delay, actions delay, and observations delay. Rewards delay has been studied extensively~\citep{han2022off,arjona2019rudder,zhang2023multi}, and in this study, we concentrate on observations and actions delay, which have also been proved to be the equivalent problem~\citep{katsikopoulos2003markov}. Early approaches apply RL techniques to the original state space. While it maintains high computational efficiency, the performance significantly deteriorates in delayed observation due to the absence of Markovian property. Subsequent improvements leveraged various predictive models like deterministic generators~\citep{learning_planning_delayed_feedback}, Gaussian models~\citep{chen2021delay}, and transformer architectures~\citep{learning_belief}. Additionally, there are also attempts building upon the world model with delay adaption~\citep{karamzade2024delay,valensi2024tree}. However, learning from the original state space cannot effectively compute approximation errors accumulated from observation delays, causing the performance to deteriorate with suboptimal solutions in delayed settings ~\citep{learning_belief}. The augmentation-based approach augments the state/action space with relevant past actions/states. This approach is notably more promising as it restores the Markovian property~\citep{altman_delay, delay_mdp, kim2023belief, wu2024boosting} but possesses inherent sample complexity/efficiency issues. Solving this issue, \cite{wu2024variational} proposes a novel framework to transfer the original delayed RL objective into a behavior cloning of the delay-free policy through variational inference techniques. In contrast to the number of delayed RL research, there are limited works in the delayed IRL setting.

\paragraph{Inverse RL.}
Early research in Inverse Optimal Control (IOC)~\citep{kalman1964linear} focused on recovering optimal control rule that maximizes margins between expert demonstrations and alternative policies, which evolves into machine-learning version Inverse Reinforcement Learning (IRL)\citep{ng2000algorithms}. Bayesian methodologies later explored varied reward priors - from Boltzmann distributions \citep{ramachandran2007bayesian, choi2011map, chan2021scalable} to Gaussian Processes \citep{levine2011nonlinear}. Concurrent development on statistical approaches expanded the field through multi-class classification \citep{klein2012inverse, brown2019extrapolating} and decision tree methods \citep{levine2010feature}. Entropy-based optimization leverages the Maximum Entropy(ME) principle \citep{shore1980axiomatic} to determine trajectory distributions from reward parameters. \cite{ziebart2008maximum,ziebart2010modeling} reformulated reward inference as a maximum likelihood problem using a linear combination of hand-crafted features. \citet{wulfmeier2015maximum} extended it to deep neural network representation, while \citet{finn2016guided} added importance sampling for model-free estimation. Inspired by Generative Adversarial Networks, the latest advances in IRL center on adversarial methods, where discriminator networks learn reward functions by distinguishing between expert and agent behaviors \citep{ho2016generative, fu2017learning}. There are extensions to solve sample efficiency \citep{kostrikov2018discriminator, blonde2019sample} and to solve stochasticity MDP issue \citep{zhan2024model}. As for IRL under delayed scenarios, there has been some attention on delayed rewards~\citep{krishnan2016hirl, krishnan2019swirl}. However, to our knowledge, few have considered delayed action and observations under IRL settings.

\section{Preliminaries}
\label{sec::prelim}

In delay-free IRL setting, we consider a standard MDP without reward function $\mathcal{M}'$ defined as a tuple $\langle \mathcal{S}, \mathcal{A}, \mathcal{T}, \gamma, \rho \rangle$, where $\mathcal{S}$ is the state space $s\in\mathcal{S}$, $\mathcal{A}$ is the action space $a\in\mathcal{A}$, $\mathcal{T}: \mathcal{S} \times \mathcal{A} \times \mathcal{S} \rightarrow [0,1]$ is the transition probability function, $\gamma \in (0,1)$ is the discount factor, and $\rho_0$ is the initial state distribution. The discounted visitation distribution of trajectory $\tau$ with policy $\pi$ is given by:
\begin{equation}
\label{eq::normal_visitation_prob}
    p(\tau)=\rho_0\prod_{t=0}^{T-1}\gamma^t\mathcal{T}(s_{t+1}\vert s_t,a_t)\pi(a_t\vert s_t),
\end{equation}
where $T$ is the horizon.
Given an MDP $\mathcal{M}'$ without reward and expert trajectories collected in a data buffer $D_{exp}= \{\tau_1, ..., \tau_n\}$ where $\tau_i$ represents individual trajectories collected using expert demonstration policy $\pi^E: \mathcal{S} \rightarrow \mathcal{A}$, IRL infers reward function $R_\theta: \mathcal{S} \times \mathcal{A} \rightarrow \mathbb{R}$, where $\theta$ is the reward parameter. Maximizing the entropy of distribution over paths subject to the feature constraints from observation ~\citep{ziebart2008maximum,ziebart2010modeling}, the optimal reward parameters are obtained by
\begin{equation}
\label{eq::irl_optimization}
    \theta^* = \arg\max_{\theta} \sum_{\mathcal{D}_{exp}} \log p(\tau|\theta).
\end{equation}
Under delayed RL settings, we consider MDPs with an observation delay between the action taken and when its state transition and reward are observed, termed delayed MDPs. Assuming under a constant observation delay $\Delta$, a delayed MDP $\mathcal{M}_{\Delta}$ inherits the Markov property based on the augmentation approaches ~\citep{altman_delay, delay_mdp}. It can be formulated as a tuple $\langle \mathcal{X}, \mathcal{A}, \mathcal{T}_\Delta, R_\Delta, \gamma, \rho_\Delta \rangle$. The augmented state space is defined as $\mathcal{X}:= \mathcal{S} \times \mathcal{A}^\Delta$, where an augmented state $x_t=\{s_{t-\Delta}, a_{t-\Delta},\cdots, a_{t-1}\} \in \mathcal{X}$. The delayed transition function is defined as:
\begin{equation}
\label{eq::delay_transition}
    \mathcal{T}_\Delta(x_{t+1}|x_t, a_t):=\mathcal{T}(s_{t-\Delta+1}|s_{t-\Delta}, a_{t-\Delta})\delta_{a_t}(a'_t)\prod_{i=1}^{\Delta-1}\delta_{a_{t-i}}(a'_{t-i}),
\end{equation}
where $\delta$ is the Dirac distribution. The delayed reward function is defined as $R_\Delta(x_t, a_t):= \mathop{\mathbb{E}}_{s_t\sim b(\cdot|x_t)}\left[R(s_t, a_t)\right]$ where $b$ is the belief function defined as:
\begin{equation}
\label{eq::belief_function}
    b(s_t|x_t):=\int_{\mathcal{S}^\Delta}\prod_{i=0}^{\Delta-1}\mathcal{T}(s_{t-\Delta+i+1}|s_{t-\Delta+i}, a_{t-\Delta+i})\mathrm{d}{s_{t-\Delta+i+1}}.
\end{equation} 
And the initial augmented state distribution is defined as $\rho_\Delta =\rho_0\prod_{i=1}^{\Delta}\delta_{a_{-i}}$. Correspondingly, we can define the trajectory visitation probability in the delayed MDP $\mathcal{M}_\Delta$ with policy $\pi_\Delta$ as:
\begin{equation}
\label{eq::delayed_visitation_prob}
    p(\tau_\Delta)=\rho_\Delta\prod_{t=0}^{T-1}\gamma^t\mathcal{T}_\Delta(x_{t+1}\vert x_t,a_t)\pi_\Delta(a_t\vert x_t).
\end{equation}
Under the delayed IRL setting, expert trajectories exhibit temporal misalignment, where the observed sequence follows the pattern \((s_{t-\Delta}, a_t, s_{t-\Delta+1}, \dots)\). This misalignment raises a fundamental question: \textbf{what kind of representation should be used in reward shaping to recover the expert policy?} Specifically, the policy and reward can be conditioned on the delayed observation state \(s_{t-\Delta}\), which directly corresponds to the provided trajectories, or on an augmented state representation \(x_t\) that accounts for the delay and incorporates additional context. The design of state representation is critical, as it directly affects the accuracy and robustness of the recovered policy. In the next section, we provide a theoretical analysis to address this question, offering insights into the trade-offs and advantages of each approach.

\section{Theoretical Analysis}
\label{sec::theoretical_analysis}
In this section, we analyze the difference in optimal performance when the same IRL algorithm is applied using either the delayed observation state or the augmented state. It is important to note that the reward functions in these two cases are also defined differently depending on different inputs. To quantify this performance difference, we first examine the discrepancy between the recovered reward functions under each state representation. We then extend this analysis to evaluate the impact on the value function, assuming the same policy optimization method is used. We assume that learned reward functions and transition dynamic functions satisfy the Lipschitz Continuity (LC) property, which is a common assumption that appears in RL setting~\citep{rachelson2010locality}. And reward functions are bounded by a maximum value $R_{\max}$. 
\begin{definition}[Lipschitz Continuous Reward Function~\citep{rachelson2010locality}]
\label{def::reward_func_lipschitz}
A reward function $R$ is $L_\mathcal{R}$-Lipschitz Continuous, if, $\forall (s_1, a_1), (s_2, a_2) \in \mathcal{S} \times \mathcal{A}$, it satisfies
\begin{equation*}
    d_\mathcal{R}(R(s_1, a_1) - R(s_2, a_2)) \leq L_\mathcal{R} (d_{\mathcal{S}}(s_1, s_2) + d_{\mathcal{A}}(a_1, a_2)).
\end{equation*}
\end{definition}

\begin{definition}[Time Lipschitz Continuous Dynamic~\citep{metelli2020control}]
\label{def::dynamic_func_lipschitz}
A dynamic is $L_\mathcal{T}$-Time Lipschitz Continuous, if, $\forall (s_1, a_1), (s_2, a_2) \in \mathcal{S} \times \mathcal{A}$, it satisfies
\begin{equation*}
    W_1(\mathcal{T}(\cdot|s, a) || \delta_s) \leq L_\mathcal{T}.
\end{equation*}
\end{definition}
where \textit{Euclidean distance} $d$ is adopted to describe distance in a deterministic space (e.g., $d_\mathcal{S}$ for state space $\mathcal{S}$, $d_\mathcal{A}$ for action space $\mathcal{A}$ and $d_\mathcal{R}$ for reward space $\mathcal{R}$), and $L1$-\textit{Wasserstein distance}~\citep{villani2009optimal}, denoted as $W_1$, is used in a probabilistic space. From the above assumptions, we can infer the Lipschitz continuity on belief function. Detailed proof can be found in \appendixref{proof::belief_func_lipschitz}.
\begin{lemma}[Time Lipschitz Continuous Belief]
\label{lemma::belief_func_lipschitz}
Given a $L_\mathcal{R}$-Time Lipschitz Continuous Dynamic, the belief $b$ is $L_\mathcal{R}$-Time Lipschitz Continuous, $\forall x_t \in \mathcal{X}$, satisfying
\begin{equation*}
    W_1(b(\cdot|x_t) || \delta_{s_{t-\Delta}}) \leq \Delta L_\mathcal{T}.
\end{equation*}
\end{lemma}

Next, we extend the continuity bounds to the learned reward functions defining on delayed observation state and augmented state respectively. Note that we assume both reward functions are recovered using the same IRL algorithm and are parameterized by MLP with ReLU activation, which satisfies the Lipschitz continuous assumption~\citep{virmaux2018lipschitz}. Details of the IRL algorithm and implementation are presented in \sectionref{sec::inverse_drl}, and the detailed proof is in \appendixref{proof:reward_bound}.

\begin{lemma}[Reward Delayed Difference Upper Bound]
\label{lemma:reward_bound}
Given a $L_\mathcal{R}$-Time Lipschitz Continuous Dynamic and $L_\mathcal{R}$-Lipschitz Continuous Reward function, $\forall x_t \in \mathcal{X}$, the upper bound of reward delayed difference is as follows:
\begin{equation*}
    d_{\mathcal{R}}\left(R_\Delta(x_t,a_t) - R(s_{t-\Delta}, a_{t})\right)\leq \Delta L_\mathcal{R} L_\mathcal{T}.
\end{equation*}
\end{lemma}

With the difference bound on the reward function, we seek to extend this bound to value functions corresponding to different policies $\pi$ and $\pi_\Delta$, which directly reflect the performance difference. We define the value function of $\pi_\Delta$ as follows:
\begin{equation}
\label{eq::value_delayed}
    V^{\pi_\Delta}(x_t) = \mathop{\mathbb{E}}_{
    x_{t+1} \sim \mathcal{T}_{\Delta}(\cdot|x_{t}, a_{t})\atop
    a_{t} \sim {\pi_\Delta}(\cdot|x_{t})
    }\left[R_\Delta(x_t,a_t)+\gamma V^{\pi_\Delta}(x_{t+1})\right].
\end{equation}
And the definition of the value function corresponding with $\pi$ is the following. 
\begin{equation}
    V^\pi(x_t) = \mathop{\mathbb{E}}_{
    x_{t+1} \sim \mathcal{T}_{\Delta}(\cdot|x_{t}, a_{t})\atop
    a_{t} \sim \pi(\cdot|s_{t-\Delta})
    }\left[R(s_{t-\Delta},a_t)+\gamma V^\pi(x_{t+1})\right].
\end{equation}

\begin{proposition}[Performance Difference Upper Bound]
\label{prop::performance_difference}
    Given a $L_\mathcal{R}$-Time Lipschitz Continuous Dynamic, $L_\mathcal{R}$-Lipschitz Continuous Reward function, and policies $\pi$ and $\pi_\Delta$, $\forall x_t \in \mathcal{X}$, the upper bound of performance difference is as follows:
    \begin{equation*}
        \Vert V^{\pi_\Delta}(x_t) - V^\pi(x_t)\Vert \leq \frac{1}{1-\gamma} \left[R_{\max} + \Delta L_\mathcal{R} L_\mathcal{T}\right].
    \end{equation*}
\end{proposition}
Detailed proof can be found in \appendixref{proof::performance_difference}. From Prop.~\ref{prop::performance_difference}, we provide a theoretical insight to choose augmented state $x$ instead of delayed observation state to recover value function. 

\section{Off-Policy Inverse Delayed RL}
In this section, we present the framework of our off-policy Inverse Delayed Reinforcement Learning (IDRL) approach. We begin by introducing the overall structure of the framework and provide a detailed explanation of the adversarial formulation for the reward function, along with the underlying intuition and adaption to the off-policy framework. Following this, we delve into the specifics of the algorithm, including the augmentation of expert trajectories with delayed observations and the policy optimization techniques employed. The overall algorithmic framework can be found in \algorithmref{algo:idrl}. 
\label{sec::inverse_drl}

\begin{algorithm*}[h]
\caption{Inverse Delayed RL}
\label{algo:idrl}
\begin{algorithmic}[1]
    \STATE Obtain expert buffer $\mathcal{D}_{exp}$.
    \STATE Initialize policy $\pi^\psi_\Delta$, discriminator $D_{\theta}$, and buffers $\mathcal{D}_{env}$. 
    \FOR{step $t$ in \{1, \dots, N\}}
    \STATE Interact with real environments and add augmented state-action pair $(x_t,a_t,x_{t+1})$ to $\mathcal{D}_{env}$.
    \STATE Sample delayed observation state-action batch $(s_{t-\Delta}, a_t, s_{t-\Delta+1},\cdots)$ from $\mathcal{D}_{exp}$.
    \STATE Augment delayed state and action sequences into augmented state action batch $(x_t, a_t, x_{t+1})$.
    \STATE Train $D_{\theta}$ via ~\equationref{eq::disc_loss} to classify expert samples from $\mathcal{D}_{exp}$ and samples from $\mathcal{D}_{env}$.
    \STATE Sample new augmented state action batch from $\mathcal{D}_{env}$.
    \STATE Freeze $\theta$ and calculate rewards of the new augmented state action batch using \equationref{eq::reward_design}.
    \STATE Conduct policy optimization procedure \algorithmref{algo:policy_optimization} with calculated rewards to update $\psi$. 
    \ENDFOR
\end{algorithmic}
\end{algorithm*}

\paragraph{Adversarial Formulation.} 
Generative Adversarial Networks (GANs)~\citep{goodfellow2020generative} inspire our adversarial framework, where a binary discriminator \(D_\theta(x_t, a_t)\) is trained to distinguish between augmented state-action samples from expert demonstrations and those generated by the imitator policy \(\pi_\Delta\), where $D_\theta$ has the following form:
\begin{equation}
    D_\theta(x,a)=\frac{\exp(R_\theta(x,a))}{\exp(R_\theta(x,a))+\pi_\Delta(a\vert x)},
\end{equation}
where $R_\theta$ is a multilayer perceptron (MLP) parameterized by $\theta$ and can be interpreted as the reward used with little modification introduced in the next paragraph. In this context, the imitator policy \(\pi_\Delta\) serves as a generator, improving itself to "fool" the discriminator by making its generated samples indistinguishable from expert samples. The discriminator is trained using the following cross-entropy loss, designed to classify samples as either coming from the expert or the imitation policy:
\begin{equation*}
    \mathcal{L}_{disc} = -\mathbb{E}_{\mathcal{D}_{exp}} \left[\log D_\theta(x, a)\right] - \underbrace{\mathbb{E}_{\pi_\Delta} \left[\log(1 - D_\theta(x, a))\right]}_A.
\end{equation*}
% \simon{
The proof of state-action occupancy match between the policy induced from the above adversarial formulation and expert policy have been shown by \cite{ho2016generative} under the on-policy fashion. The proof sketch should be similar to our modified adversarial formulation with an extension to the augmented state. In the following, we elaborate on the extension to off-policy setting and additional loss terms introduced to stabilize GAN training.
% }
To enable an off-policy fashion, importance sampling must be applied to the second term \(A\), resulting in $\mathbb{E}_{\pi_\Delta} \left[\frac{p_{\pi^\omega_\Delta}(x, a)}{p_{\pi_\Delta}(x, a)} \log(1 - D_\theta(x, a))\right]$. However, estimating this density ratio is computationally challenging. Additionally, omitting this term has been observed to improve the algorithm's performance in practice, potentially due to reduced gradient variance during updates~\citep{neal2001annealed}. Despite this improvement, training using cross-entropy loss alone often results in instability due to the complex interaction between the discriminator and the generator policy, which is also updated every iteration. To address these issues, we incorporate additional regularization terms: a gradient penalty \(\mathcal{L}_{grad}\) and an entropy regularization term \(\mathcal{L}_{entropy}\). These modifications help stabilize training by preventing excessively large gradient updates per training epoch~\citep{nagarajan2017gradient,arjovsky2017wasserstein}. The final loss function for the discriminator is defined as:
\begin{equation}
\label{eq::disc_loss}
    \mathcal{L}_{disc} = -\mathbb{E}_{\mathcal{D}_{exp}} \left[\log D_\theta(x, a)\right] - \mathbb{E}_{\mathcal{D}_{gen}} \left[\log(1 - D_\theta(x, a))\right] + \mathcal{L}_{grad} + \mathcal{L}_{entropy}.
\end{equation}

\paragraph{Policy Optimization.}
In this section, we introduce the reward function derived from the discriminator $D$ and the policy optimization method used to improve policy in each iteration. To extract the reward signal from the discriminator for each policy update, we use the following formula:
\begin{equation}
\label{eq::reward_design}
    \hat{R}_\theta(x,a) = \log(D_\theta(x,a)+\delta) - \log(1-D_\theta(x,a)+\delta),
\end{equation}
where $\delta$ is a marginal constant to prevent numerical error in computation. After derivation $\hat{R}_\theta$ resembles policy entropy regularized reward for soft policy update, which can be used to satisfy the delayed RL objectives $\max\mathbb{E}_{\tau_\Delta\sim p(\tau_\Delta)}\left[\sum_{t=0}^{T-1}\gamma^t R_\theta(\tau_\Delta)-H(\pi_\Delta)\right]$ ~\citep{haarnoja2018soft}. To update policy $\pi_\Delta$ for each iteration, we also apply the augmented approach introduced in \sectionref{sec::related_work}. However, using state augmentation for both reward learning and policy optimization is sample inefficient. Thus, we apply the auxiliary delayed RL approach, which learns a value function for
short delays and uses bootstrapping and policy
improvement techniques to adjust it for long delays~\citep{wu2024boosting}. The detailed version of the algorithm can be found in \appendixref{appendix::algorithm}.

\section{Experiment}
\label{sec::experiment}
In this section, we evaluate the performance of our Off-Policy Inverse Delayed RL framework. We aim to demonstrate the capability of our method to recover expert behaviors under delayed settings. All experiments are conducted on the MuJoCo benchmarks~\citep{mujoco}. 
All the \textbf{expert trajectories} are collected by an expert agent trained with VDPO~\citep{wu2024variational} under MuJoCo environments with \textbf{5, 10, and 25 delay steps}.
We compare our approach with the on-policy algorithms AIRL \citep{fu2017learning}, behavior cloning~\citep{torabi2018generative}, and the off-policy method Discriminator Actor-Critic (DAC) \citep{kostrikov2018discriminator} based on delayed observation states. For policy optimization to IRL approaches, we use Proximal Policy Optimization (PPO) \citep{schulman2017proximal} for AIRL, and Soft Actor-Critic (SAC) \citep{haarnoja2018soft} for DAC. All implementations of PPO and SAC are referenced from the \texttt{Clean RL} library~\citep{huang2022cleanrl}. Each algorithm is trained with 100k environmental steps and evaluated each 1k steps across 5 different seeds for \texttt{InvertedPendulum-v4}. For \texttt{Hopper-v4}, \texttt{HalfCheetah-v4}, \texttt{Walker2d-v4}, and \texttt{Ant-v4}, AIRL is trained with 10M steps and evaluated every 100k steps across 5 different seeds, but DAC and our algorithm are trained with 1M environmental steps and evaluated every 10k steps across 5 different seeds. We conduct the aforementioned series of experiments under \textbf{various numbers of expert trajectories ranging from 10 to 1000}.  All the experiments are run on the Desktop equipped with RTX 4090 and Core-i9 13900K. Training graphs are provided in \appendixref{appendix::training_curves}. Across different scenarios, we showcase our method's efficacy in two dimensions.
\begin{itemize}
    \item \textbf{Performance Superiority}: 
    Our method consistently outperforms baseline approaches in recovering expert behaviors across various environments, even under diverse delay-length settings. When expert demonstrations are sufficient, our method achieves near-complete recovery of expert performance, whereas most baselines fail to learn meaningful policies.
    \item \textbf{Robust Utilization of Limited Expert Demonstrations}: 
    Even with limited expert demonstrations, our method demonstrates the ability to recover a reasonable policy, outperforming baselines that struggle to reproduce expert behaviors in most environments.
\end{itemize}

\subsection{Impact of Varying Delays}

We investigate the effect of varying delays (5, 10, and 25 delay steps) on performance, using 1000 expert demonstration trajectories. Detailed learning curves in \appendixref{appendix::training_curves} support our findings. As shown in \tableref{tab:delay}, our IDRL framework consistently outperforms baseline methods (AIRL, DAC, and BC) across all tested MuJoCo environments. Notably, as delays increase, the performance of the expert policy deteriorates, leading to noisier expert trajectories and a corresponding decline in baseline performance. While baseline methods exhibit significant degradation or fail entirely in certain environments, IDRL maintains near-expert performance across all delay conditions, demonstrating remarkable robustness. This resilience is particularly evident in complex environments like \texttt{Ant-v4} and \texttt{Walker2d-v4}, where IDRL consistently achieves superior performance, whereas DAC and AIRL fail to adapt. BC can recover part of the expert behaviors, but still a large margin below ours. These advantages are even more notable in the low dimensional task (\texttt{InvertedPendulum-v4}) and medium dimensional tasks (\texttt{Hopper-v4} and \texttt{HalfCheetah-v4}). IDRL’s robustness stems from its augmented state representation, which enriches the feature space to better capture delayed dependencies, combined with advanced policy optimization techniques that mitigate the adverse effects of delays. These results empirically validate our theoretical analysis, highlighting the effectiveness of leveraging augmented state representations over direct delayed observations in handling delay-affected environments. We also observe randomness in the performance of some algorithms under certain scenarios, likely due to their inability to extract meaningful reward features from expert demonstrations. Additionally, BC exhibits inconsistent performance trends in some environments. Since BC directly replicates the action distribution of expert demonstrations, any noise present in the expert trajectories—potentially introduced during data collection—can severely degrade the imitator's performance.

\begin{table}[tb]
\centering
\small
\setlength{\tabcolsep}{6pt} % Adjust column spacing
\renewcommand{\arraystretch}{1.2} % Adjust row spacing
\caption{Performance comparison across environments and algorithms with 1000 expert demonstration trajectories under varying delay steps from 5 to 25. Results are mean \(\pm\) standard deviation. Best performances are highlighted in \textcolor{blue}{blue}. We omit \texttt{InvertedPendulum-v4} under 25 delays, since the expert policy degrades to performance near random policy.}
\scalebox{0.8}{
\begin{tabular}{@{}lccrrrrr@{}}
\toprule
\textbf{Task} & \textbf{Delay} & \textbf{Expert} & \textbf{BC} & \textbf{AIRL} & \textbf{DAC} & \textbf{IDRL (Ours)} \\
\midrule
\multirow{2}{*}{\texttt{InvertedPendulum-v4}} 
& 5  & 974.29\scriptsize{$\pm$157.44}   & 15.27\scriptsize{$\pm$2.11}     & 28.93\scriptsize{$\pm$5.28}     & 27.80\scriptsize{$\pm$20.28}    & \textcolor{blue}{1000.00\scriptsize{$\pm$0.00}} \\ 
& 10 & 681.11\scriptsize{$\pm$462.73}   & 21.06\scriptsize{$\pm$6.16}     & 28.53\scriptsize{$\pm$1.59}     & 23.00\scriptsize{$\pm$7.72}     & \textcolor{blue}{867.87\scriptsize{$\pm$186.87}} \\ 
\midrule
\multirow{3}{*}{\texttt{Hopper-v4}} 
& 5  & 3738.91\scriptsize{$\pm$34.63}   & 176.67\scriptsize{$\pm$43.35}   & 203.26\scriptsize{$\pm$113.29}  & 516.88\scriptsize{$\pm$364.13}  & \textcolor{blue}{3569.99\scriptsize{$\pm$44.33}} \\ 
& 10 & 3492.25\scriptsize{$\pm$239.45}  & 14.15\scriptsize{$\pm$4.46}    & 182.52\scriptsize{$\pm$50.31}   & 120.28\scriptsize{$\pm$60.35}   & \textcolor{blue}{3321.84\scriptsize{$\pm$50.74}} \\ 
& 25 & 2107.44\scriptsize{$\pm$1399.19} & 101.32\scriptsize{$\pm$50.67}   & 182.64\scriptsize{$\pm$11.02}   & 96.96\scriptsize{$\pm$15.06}    & \textcolor{blue}{1814.18\scriptsize{$\pm$756.36}} \\ 
\midrule
\multirow{3}{*}{\texttt{HalfCheetah-v4}} 
& 5  & 5451.92\scriptsize{$\pm$239.91}  & 2384.60\scriptsize{$\pm$563.00} & 0.05\scriptsize{$\pm$0.11} & -220.04\scriptsize{$\pm$285.61} & \textcolor{blue}{4561.01\scriptsize{$\pm$313.91}} \\ 
& 10 & 4986.07\scriptsize{$\pm$852.61}  & 793.87\scriptsize{$\pm$973.06} & 0.05\scriptsize{$\pm$0.12} & -234.68\scriptsize{$\pm$85.08}  & \textcolor{blue}{5061.02\scriptsize{$\pm$154.63}} \\ 
& 25 & 4088.53\scriptsize{$\pm$1600.44} & 1087.04\scriptsize{$\pm$319.38} & 0.05\scriptsize{$\pm$0.13} & -225.55\scriptsize{$\pm$146.12} & \textcolor{blue}{3256.81\scriptsize{$\pm$693.51}} \\ 
\midrule
\multirow{3}{*}{\texttt{Walker2d-v4}} 
& 5  & 4124.08\scriptsize{$\pm$1289.46} & 1039.87\scriptsize{$\pm$389.39} & 146.64\scriptsize{$\pm$45.33}   & 812.51\scriptsize{$\pm$176.26}  & \textcolor{blue}{4424.19\scriptsize{$\pm$138.03}} \\ 
& 10 & 4491.65\scriptsize{$\pm$610.81}  & 763.85\scriptsize{$\pm$767.61} & 136.87\scriptsize{$\pm$99.16} & 315.09\scriptsize{$\pm$436.99} & \textcolor{blue}{4283.64\scriptsize{$\pm$105.36}} \\ 
& 25 & 1955.69\scriptsize{$\pm$1458.62} & 604.07\scriptsize{$\pm$277.71}  & 115.31\scriptsize{$\pm$27.66}   & 60.91\scriptsize{$\pm$72.40}    & \textcolor{blue}{1437.88\scriptsize{$\pm$506.97}} \\ 
\midrule
\multirow{3}{*}{\texttt{Ant-v4}} 
& 5  & 5281.73\scriptsize{$\pm$1627.50} & 761.11\scriptsize{$\pm$107.30}  & 1003.40\scriptsize{$\pm$2.09}  & -52.27\scriptsize{$\pm$12.55} & \textcolor{blue}{5764.42\scriptsize{$\pm$71.72}} \\ 
& 10 & 3618.59\scriptsize{$\pm$868.75}  & 799.43\scriptsize{$\pm$138.88} & 1004.32\scriptsize{$\pm$1.10}  & -62.64\scriptsize{$\pm$6.27}  & \textcolor{blue}{3949.62\scriptsize{$\pm$31.93}} \\ 
& 25 & 3432.42\scriptsize{$\pm$580.22}  & 698.95\scriptsize{$\pm$20.66}   & 1003.21\scriptsize{$\pm$3.04}  & -40.43\scriptsize{$\pm$27.07} & \textcolor{blue}{3024.53\scriptsize{$\pm$150.83}} \\ 
\bottomrule
\label{tab:delay}
\end{tabular}
}
\end{table}

\subsection{Quantity of Expert Demonstrations}
We analyze the impact of the quantity of expert demonstrations on training performance under a moderate 10 delay steps, with results summarized in Table~\ref{tab:exp_demo}. Across all methods, an overall increasing trend in return value is observed as the number of expert trajectories increases (10, 100, and 1000). In \texttt{InvertedPendulum-v4}, IDRL consistently performs at the expert level with performance increases according to the increase in quantity of expert demonstrations, demonstrating its resilience and efficiency regardless of the number of expert trajectories, while all the other baselines fail to recover meaningful expert behaviors. In \texttt{Hopper-v4}, IDRL similarly scales to expert-level performance as the quantity of demonstrations increases. Though there are notable gaps with different available quantity of expert demonstrations, our method still outperforms all baselines even with fewer demonstrations. We can observe a similar trend in \texttt{Walker2d-v4}, where our method also possesses a significant performance margin compared to all the other baselines. While BC initially outperforms IDRL with limited demonstrations (10 and 100 trajectories) in \texttt{HalfCheetah-v4}, IDRL demonstrates superior scalability as more expert data becomes available. With 1000 trajectories, IDRL significantly surpasses all baselines in \texttt{HalfCheetah-v4}, highlighting the effectiveness of its augmented state representation in leveraging larger datasets. In \texttt{Ant-v4}, BC and our method have competitive performance when demonstrations are limited, but a significant margin when more expert demonstrations become available. These results emphasize that the augmented states employed by IDRL not only enhance scalability towards expert-level performance but also ensure robust learning under delay-affected settings, particularly as the availability of expert data increases.

\begin{table}[tb]
\centering
\small
\setlength{\tabcolsep}{6pt} % Adjust column spacing
\renewcommand{\arraystretch}{1.2} % Adjust row spacing
\caption{Performance comparison across different environments and algorithms with a 10 delay steps under varying quantity of expert trajectories from 10 to 1000. Results are mean \(\pm\) standard deviation. Best performances are highlighted in \textcolor{blue}{blue}.}
\scalebox{0.8}{
\begin{tabular}{@{}lccrrrrr@{}}
\toprule
\textbf{Task} & \#\textbf{Traj} & \textbf{Expert} & \textbf{BC} & \textbf{AIRL} & \textbf{DAC} & \textbf{IDRL (Ours)} \\
\midrule
\multirow{3}{*}{\texttt{InvertedPendulum-v4}} 
& 10   & 406.40\scriptsize{$\pm$484.67} & 25.77\scriptsize{$\pm$3.39}    & 29.13\scriptsize{$\pm$2.88}    & 21.47\scriptsize{$\pm$4.22}    & \textcolor{blue}{934.07\scriptsize{$\pm$93.24}} \\ 
& 100  & 673.29\scriptsize{$\pm$465.53} & 23.38\scriptsize{$\pm$4.51}    & 28.73\scriptsize{$\pm$4.07}    & 27.07\scriptsize{$\pm$5.01}    & \textcolor{blue}{802.13\scriptsize{$\pm$161.59}} \\ 
& 1000 & 681.11\scriptsize{$\pm$462.73} & 21.06\scriptsize{$\pm$6.16}    & 28.53\scriptsize{$\pm$1.59}    & 23.00\scriptsize{$\pm$7.72}    & \textcolor{blue}{867.87\scriptsize{$\pm$186.87}} \\ 
\midrule
\multirow{3}{*}{\texttt{Hopper-v4}} 
& 10   & 3567.45\scriptsize{$\pm$64.08}   & 149.31\scriptsize{$\pm$21.43}   & 198.56\scriptsize{$\pm$59.59}   & 114.83\scriptsize{$\pm$89.47}   & \textcolor{blue}{1008.50\scriptsize{$\pm$12.30}} \\ 
& 100  & 3497.54\scriptsize{$\pm$193.82}  & 125.04\scriptsize{$\pm$48.81}   & 188.51\scriptsize{$\pm$64.77}   & 99.21\scriptsize{$\pm$36.76}    & \textcolor{blue}{1715.82\scriptsize{$\pm$1006.63}} \\ 
& 1000 & 3492.25\scriptsize{$\pm$239.45}  & 14.15\scriptsize{$\pm$4.46}    & 182.52\scriptsize{$\pm$50.31}   & 120.28\scriptsize{$\pm$60.35}   & \textcolor{blue}{3321.84\scriptsize{$\pm$50.74}} \\ 
\midrule
\multirow{3}{*}{\texttt{HalfCheetah-v4}} 
& 10   & 5171.72\scriptsize{$\pm$580.66}  & -58.58\scriptsize{$\pm$257.86} & 0.05\scriptsize{$\pm$0.13} & -197.25\scriptsize{$\pm$209.73} & \textcolor{blue}{41.28\scriptsize{$\pm$70.15}} \\ 
& 100  & 4919.62\scriptsize{$\pm$865.51}  & -17.68\scriptsize{$\pm$218.00} & 0.05\scriptsize{$\pm$0.12} & -198.83\scriptsize{$\pm$67.36}  & \textcolor{blue}{-10.68\scriptsize{$\pm$6.06}} \\ 
& 1000 & 4986.07\scriptsize{$\pm$852.61}  & 793.87\scriptsize{$\pm$973.06} & 0.05\scriptsize{$\pm$0.12} & -234.68\scriptsize{$\pm$85.08}  & \textcolor{blue}{5061.02\scriptsize{$\pm$154.63}} \\ 
\midrule
\multirow{3}{*}{\texttt{Walker2d-v4}} 
& 10   & 4578.05\scriptsize{$\pm$31.78} & 142.05\scriptsize{$\pm$122.64}  & 145.65\scriptsize{$\pm$110.57}  & 342.81\scriptsize{$\pm$359.79}  & \textcolor{blue}{1015.10\scriptsize{$\pm$76.16}} \\ 
& 100  & 4449.56\scriptsize{$\pm$723.44}  & 90.02\scriptsize{$\pm$107.01} & 139.09\scriptsize{$\pm$103.83}  & 389.67\scriptsize{$\pm$469.74}  & \textcolor{blue}{1146.64\scriptsize{$\pm$1002.86}} \\ 
& 1000 & 4491.65\scriptsize{$\pm$610.81} & 763.85\scriptsize{$\pm$767.61} & 136.87\scriptsize{$\pm$99.16}  & 315.09\scriptsize{$\pm$436.99}  & \textcolor{blue}{4283.64\scriptsize{$\pm$105.36}} \\ 
\midrule
\multirow{3}{*}{\texttt{Ant-v4}} 
& 10   & 3187.90\scriptsize{$\pm$1263.76} & 758.24\scriptsize{$\pm$367.63} & 1005.22\scriptsize{$\pm$0.63}  & -46.57\scriptsize{$\pm$21.93}  & \textcolor{blue}{932.69\scriptsize{$\pm$3.28}} \\ 
& 100  & 3598.60\scriptsize{$\pm$825.72}  & 848.39\scriptsize{$\pm$216.35}  & 1003.04\scriptsize{$\pm$1.78}  & -42.47\scriptsize{$\pm$12.48}  & \textcolor{blue}{920.06\scriptsize{$\pm$7.74}} \\ 
& 1000 & 3618.59\scriptsize{$\pm$868.75}  & 799.43\scriptsize{$\pm$138.88} & 1004.32\scriptsize{$\pm$1.10}  & -62.64\scriptsize{$\pm$6.27}   & \textcolor{blue}{3949.62\scriptsize{$\pm$31.93}} \\ 
\bottomrule
\label{tab:exp_demo}
\end{tabular}
}
\end{table}

\vspace{-0.2cm}

\section{Conclusion}
\label{sec::conclusion}

In this work, we introduced an Inverse Delayed Reinforcement Learning (IDRL) framework that effectively addresses the challenges of learning from expert demonstrations under delayed settings. By leveraging augmented state representations and advanced policy optimization techniques, our method demonstrated superior performance across a range of environments and delay conditions, consistently outperforming baseline methods. Our empirical results validated the theoretical advantages of using augmented states over direct delayed observations, highlighting the robustness and scalability of IDRL, particularly when faced with varying delays and quantities of expert data. While our approach significantly advances state-of-the-art imitation learning with delays, several avenues for future research remain open. First, extending IDRL to more complex, high-dimensional environments, such as those involving real-world robotics, would further test its robustness and generalizability. Additionally, exploring methods to reduce the dependency on large quantities of expert demonstrations, such as leveraging transfer learning or combining imitation learning with self-supervised approaches, could broaden the applicability of IDRL in data-scarce scenarios. Finally, integrating IDRL with online adaptation mechanisms to dynamically handle non-stationary delays during deployment represents an exciting direction for future work.
%%%%%%%%%%%%%%%%%%%%%%%%%%%%%%%%%%%%%%%%%%%%%%%%%%%%%%%%%%%%%%%%%%%%%%%%%%%%%%%

\acks{We thank a bunch of people.}
\clearpage
\bibliography{citation}
%%%%%%%%%%%%%%%%%%%%%%%%%%%%%%%%%%%%%%%%%%%%%%%%%%%%%%%%%%%%%%%%%%%%%%%%%%%%%%%
\newpage
\appendix
\section{Proof}
\begin{lemma}[Time Lipschitz Continuous Belief]
Given a $L_\mathcal{R}$-Time Lipschitz Continuous Dynamic, the belief $b$ is $L_\mathcal{R}$-Time Lipschitz Continuous, $\forall x_t \in \mathcal{X}$, satisfying
$$
\begin{aligned}
    W_1(b(\cdot|x_t) || \delta_{s_{t-\Delta}}) \leq \Delta L_\mathcal{T}.
\end{aligned}
$$
\end{lemma}

\begin{proof}
\label{proof::belief_func_lipschitz}
We dentoe that $b^{(\Delta)}(\cdot|x_t) = b(\cdot|x_t)$ and $b^{(1)}(\cdot|x_t) = \mathcal{T}(\cdot|s_{t-\Delta}, a_{t-\Delta})$.
$$
    \begin{aligned}
        W_1(b(\cdot|x_t) || \delta_{s_{t-\Delta}}) &\leq W_1(b^{(\Delta)}(\cdot|x_t) || b^{(\Delta-1)}(\cdot|x_t)) + W_1(b^{(\Delta-1)}(\cdot|x_t) || \delta_{s_{t-\Delta}})\\
        & \leq L_\mathcal{T} + W_1(b^{(\Delta-1)}(\cdot|x_t) || b^{(\Delta-2)}(\cdot|x_t)) + W_1(b^{(\Delta-2)}(\cdot|x_t) || \delta_{s_{t-\Delta}})\\
        & \cdots\\
        & \leq (\Delta - 1)L_\mathcal{T} + W_1(b^{(1)}(\cdot|x_t) || \delta_{s_{t-\Delta}})\\
        & \leq \Delta L_\mathcal{T}\\
    \end{aligned}
$$  
\end{proof}

\begin{lemma}[Reward Delayed Difference Upper Bound]
Given a $L_\mathcal{R}$-Time Lipschitz Continuous Dynamic and $L_\mathcal{R}$-Lipschitz Continuous Reward function, $\forall x_t \in \mathcal{X}$, the upper bound of reward delayed difference is as follows:
\begin{equation*}
    d_{\mathcal{R}}\left(\mathop{\mathbb{E}}_{s_t\sim b(\cdot|x_t)}\left[\mathcal{R}(s_t, a_t)\right] - \mathcal{R}(s_{t-\Delta}, a_{t})\right)\leq \Delta L_\mathcal{R} L_\mathcal{T}
\end{equation*}
\end{lemma}

\begin{proof}
\label{proof:reward_bound}
We substitute the delayed reward function $R_\Delta(x_t,a_t)$ to the belief version $\mathop{\mathcal{E}}_{s_t\sim b(\cdot\vert x_t)}[R(s_t,a_t)]$
$$
    \begin{aligned}
    d_{\mathcal{R}}\left(\mathop{\mathbb{E}}_{s_t\sim b(\cdot|x_t)}\left[\mathcal{R}(s_t, a_t)\right] - \mathcal{R}(s_{t-\Delta}, a_{t})\right)&\leq L_\mathcal{R} \mathop{\mathbb{E}}_{s_{t}\sim b(\cdot|x_{t})} \left[d_{\mathcal{S}}(s_t, s_{t-\Delta})\right]\\ 
    &\leq L_\mathcal{R} W_1(b(\cdot|x_{t})||\delta_{s_{t-\Delta}})\\ 
    &\leq \Delta L_\mathcal{R} L_\mathcal{T}.\\ 
\end{aligned}
$$
\end{proof}

\begin{proposition}[Performance Difference Upper Bound]
Given a $L_\mathcal{R}$-Time Lipschitz Continuous Dynamic, $L_\mathcal{R}$-Lipschitz Continuous Reward function, and policies $\pi$ and $\pi_\Delta$, $\forall x_t \in \mathcal{X}$, the upper bound of performance difference is as follows:
\begin{equation*}
    \Vert V^{\pi_\Delta}(x_t) - V^\pi(x_t)\Vert \leq \frac{1}{1-\gamma} \left[R_{\max} + \Delta L_\mathcal{R} L_\mathcal{T}\right]
\end{equation*}
\end{proposition}
\begin{proof}
\label{proof::performance_difference}
$$
\begin{aligned}
    &V^{\pi_\Delta}(x_t) - V^\pi(x_t)\\
    =& V^{\pi_\Delta}(x_t) 
    - \mathop{\mathbb{E}}_{
    x_{t+1} \sim \mathcal{T}_{\Delta}(\cdot|x_{t}, a_{t})\atop
    a_{t} \sim \pi(\cdot|x_t)
    }\left[\mathop{\mathbb{E}}_{s_{t} \sim b(\cdot|x_{t})}\left[\mathcal{R}(s_t,a_t)\right]+\gamma V^{\pi_\Delta}(x_{t+1})\right]\\
    &+ \mathop{\mathbb{E}}_{
    x_{t+1} \sim \mathcal{T}_{\Delta}(\cdot|x_{t}, a_{t})\atop
    a_{t} \sim \pi(\cdot|x_t)
    }\left[\mathop{\mathbb{E}}_{s_{t} \sim b(\cdot|x_{t})}\left[\mathcal{R}(s_t,a_t)\right]+\gamma V^{\pi_\Delta}(x_{t+1})\right] 
    - V^\pi(x_t)\\
    =& V^{\pi_\Delta}(x_t) 
    - \mathop{\mathbb{E}}_{
    a_{t} \sim \pi(\cdot|x_t)
    }\left[Q^{\pi_\Delta}(x_t, a_t)\right]\\
    &+
    \mathop{\mathbb{E}}_{
    x_{t+1} \sim \mathcal{T}_{\Delta}(\cdot|x_{t}, a_{t})\atop
    a_{t} \sim \pi(\cdot|x_t)
    }\left[\mathop{\mathbb{E}}_{s_{t} \sim b(\cdot|x_{t})}\left[\mathcal{R}(s_t,a_t)\right] - \mathcal{R}(s_t,a_t)\right]
    + 
    \gamma\mathop{\mathbb{E}}_{
    x_{t+1} \sim \mathcal{T}_{\Delta}(\cdot|x_{t}, a_{t})\atop
    a_{t} \sim \pi(\cdot|x_t)
    }\left[V^{\pi_\Delta}(x_{t+1}) - V^\pi(x_{t+1})\right]\\
\end{aligned}
$$
Therefore,
$$
\begin{aligned}
    &|| V^{\pi_\Delta}(x_t) - V^\pi(x_t) ||\\
    \leq& ||V^{\pi_\Delta}(x_t) 
    - \mathop{\mathbb{E}}_{
    a_{t} \sim \pi(\cdot|x_t)
    }\left[Q^{\pi_\Delta}(x_t, a_t)\right]||\\
    &+
    ||\mathop{\mathbb{E}}_{
    x_{t+1} \sim \mathcal{T}_{\Delta}(\cdot|x_{t}, a_{t})\atop
    a_{t} \sim \pi(\cdot|x_t)
    }\left[\mathop{\mathbb{E}}_{s_{t} \sim b(\cdot|x_{t})}\left[\mathcal{R}(s_t,a_t)\right] - \mathcal{R}(s_t,a_t)\right]||\\
    &+ 
    \gamma||\mathop{\mathbb{E}}_{
    x_{t+1} \sim \mathcal{T}_{\Delta}(\cdot|x_{t}, a_{t})\atop
    a_{t} \sim \pi(\cdot|x_t)
    }\left[V^{\pi_\Delta}(x_{t+1}) - V^\pi(x_{t+1})\right]||\\
    \leq&
    \frac{1}{1-\gamma} \left[\mathcal{R}_{\max} + \Delta L_\mathcal{R} L_\mathcal{T}\right]
    \\
\end{aligned}
$$
\end{proof}

\clearpage
\section{Learning Curves}
\label{appendix::training_curves}

\begin{figure}[h]
    \centering
    \centerline{
        \subfigure[Delay=5, $\#$Traj=10]{\includegraphics[width=0.33\linewidth]{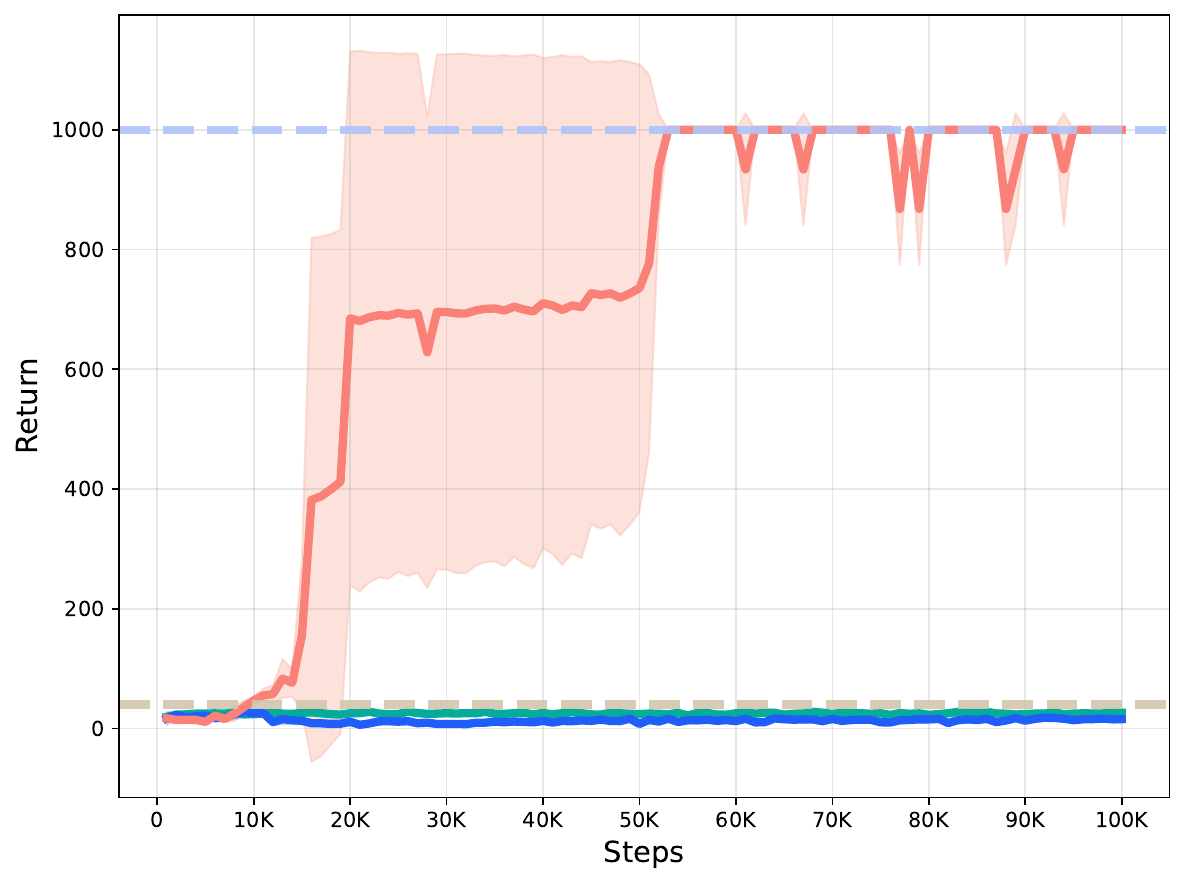}}
        \subfigure[Delay=5, $\#$Traj=100]{\includegraphics[width=0.33\linewidth]{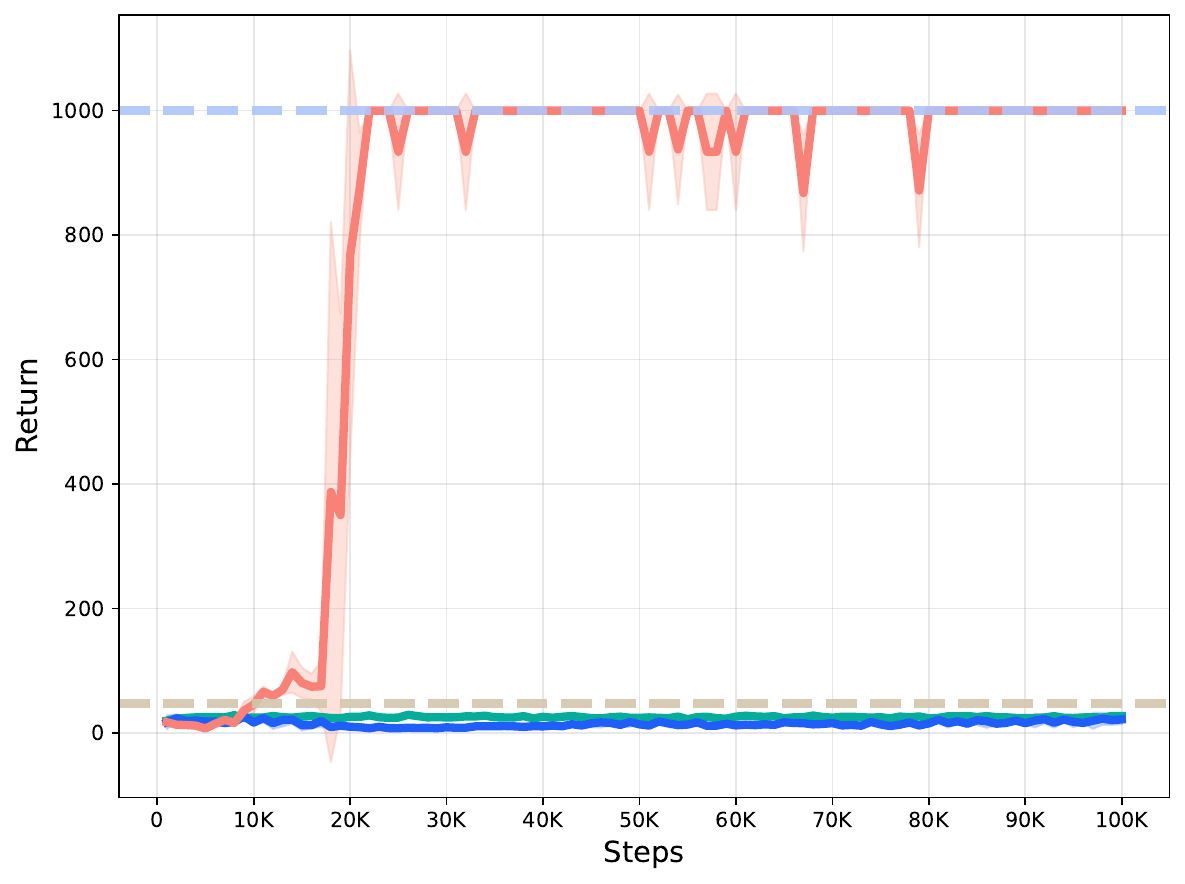}}
        \subfigure[Delay=5, $\#$Traj=1000]{\includegraphics[width=0.33\linewidth]{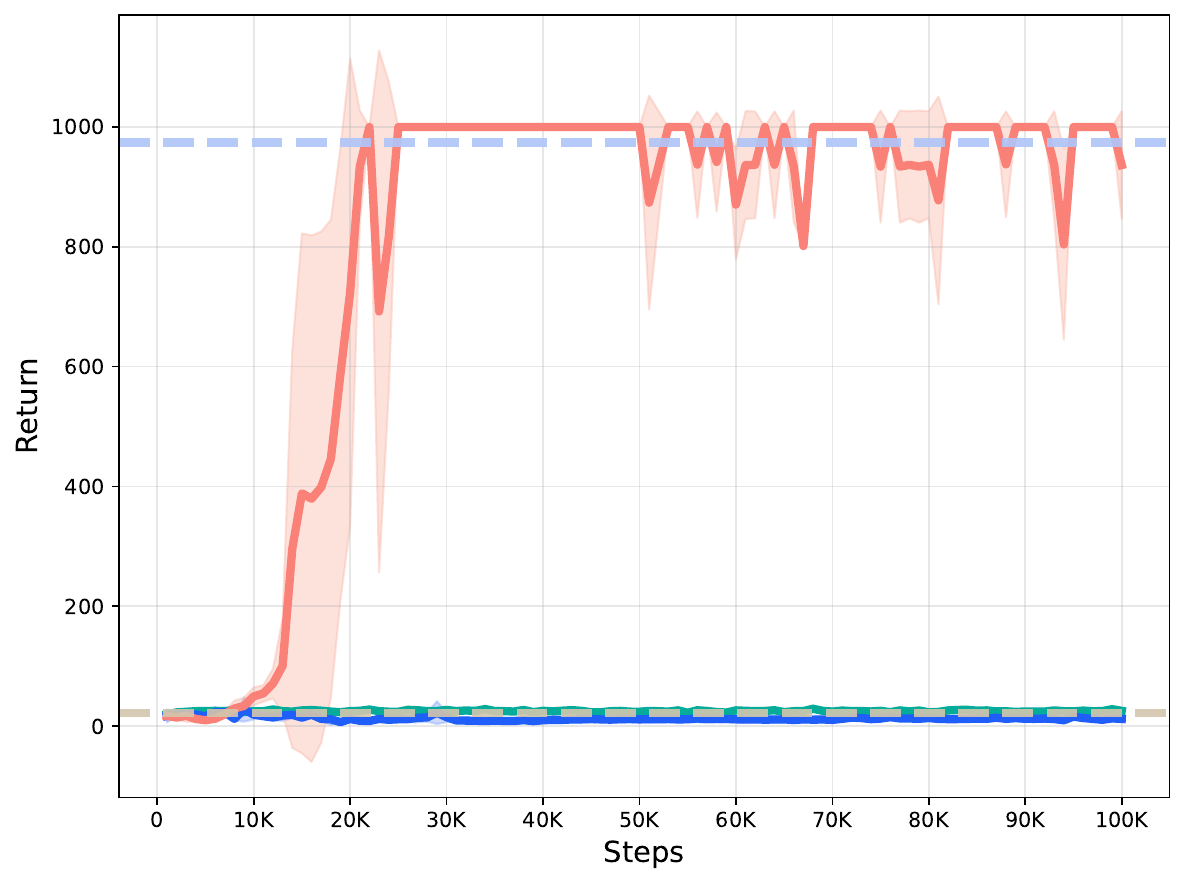}}
    }
    \centerline{
        \subfigure[Delay=10, $\#$Traj=10]{\includegraphics[width=0.33\linewidth]{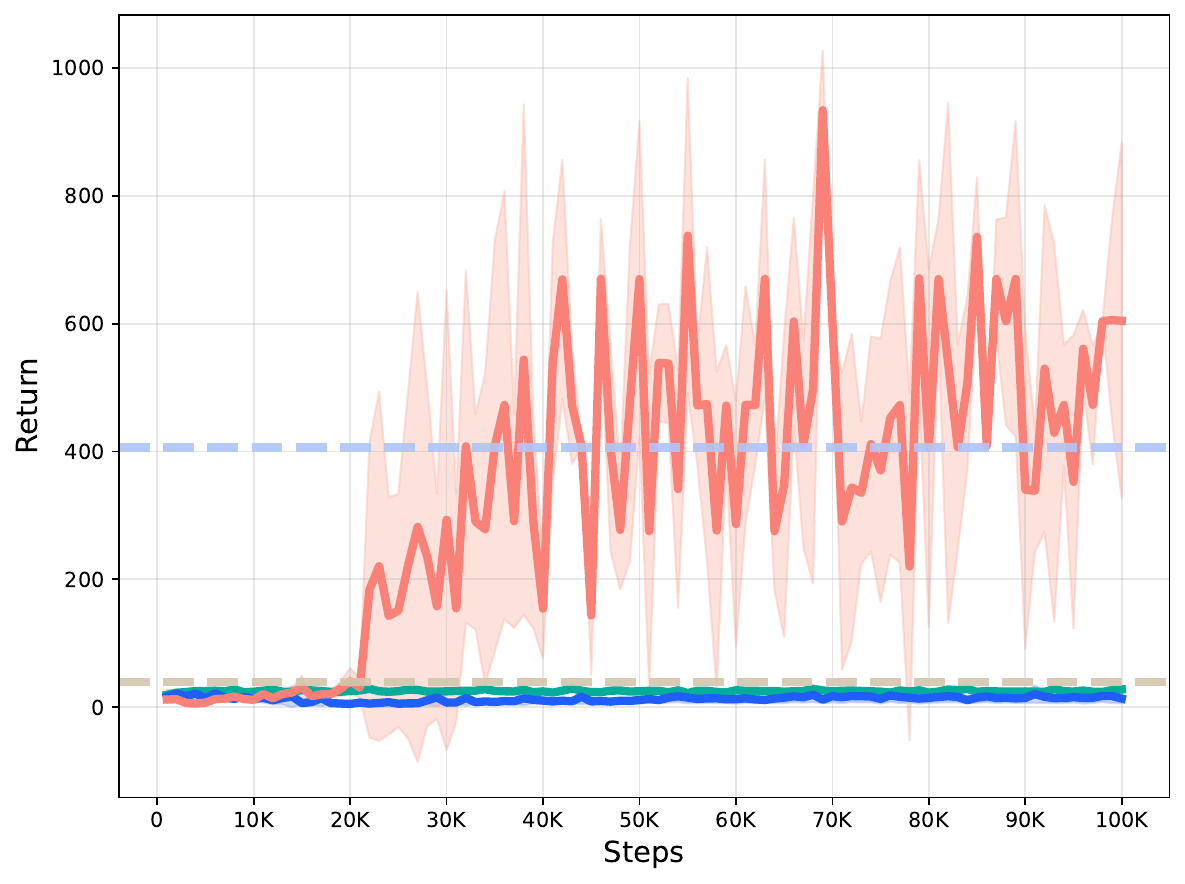}}
        \subfigure[Delay=10, $\#$Traj=100]{\includegraphics[width=0.33\linewidth]{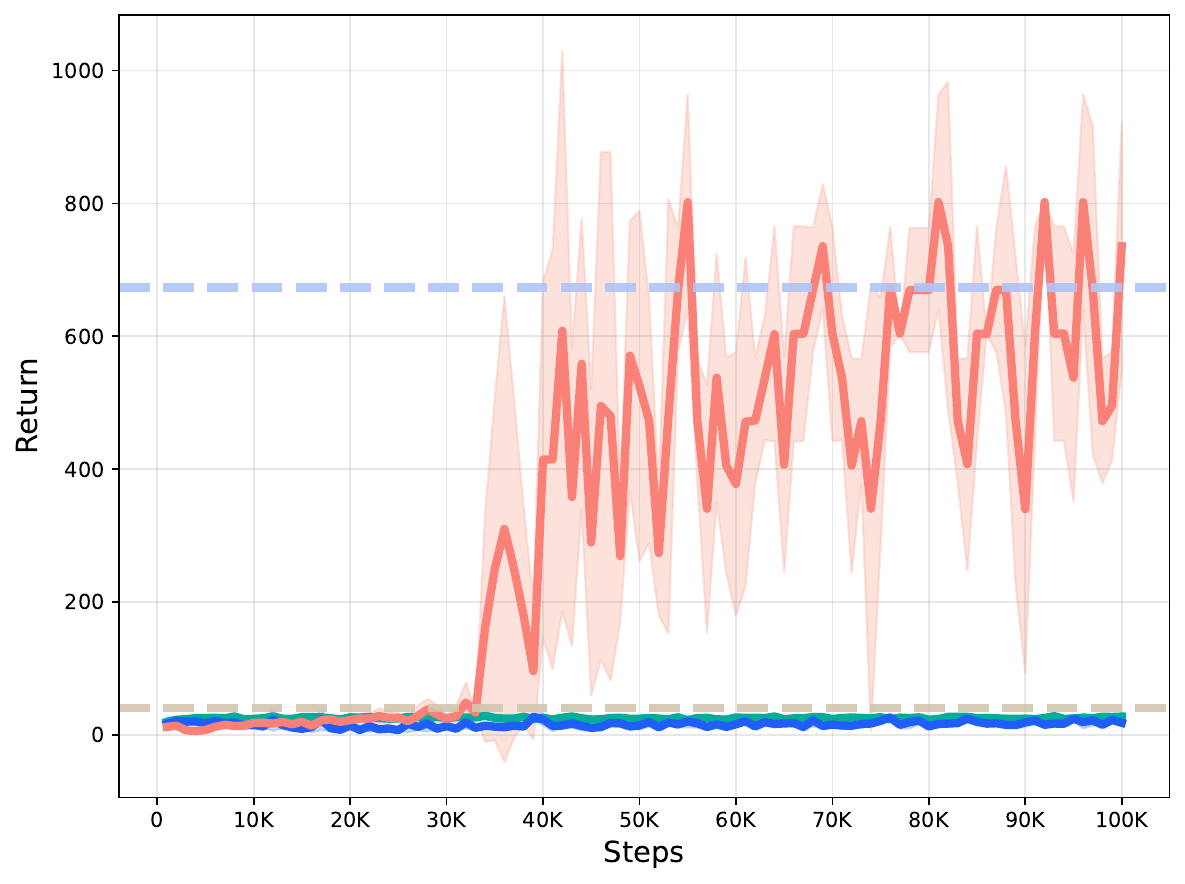}}
        \subfigure[Delay=10, $\#$Traj=1000]{\includegraphics[width=0.33\linewidth]{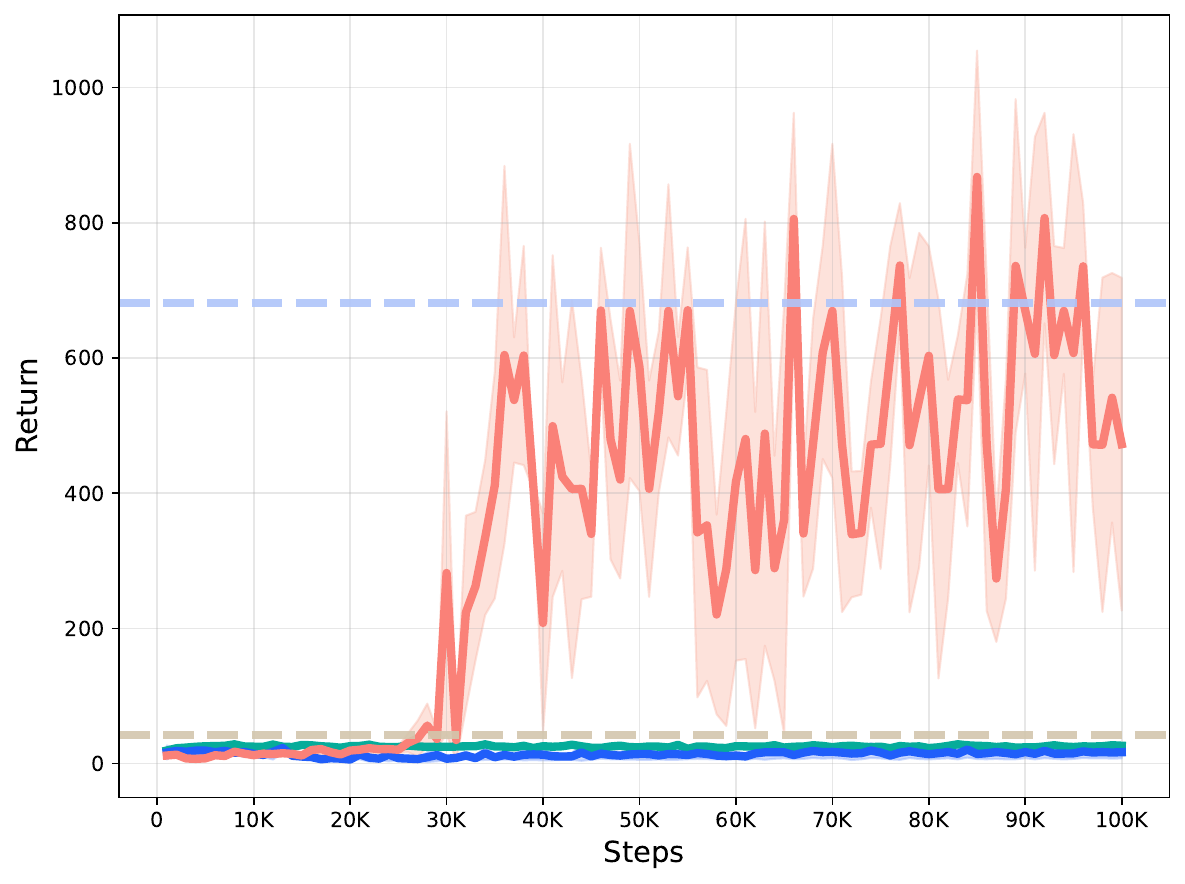}}
    }
    \centerline{
        \includegraphics[width=0.8\linewidth]{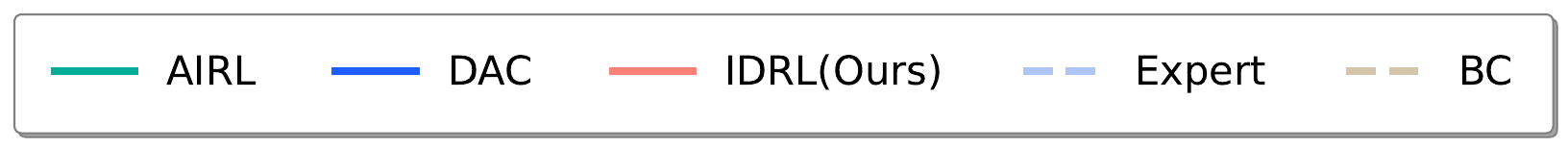}
    }
    \caption{Learning Curves on \texttt{InvertedPendulum-v4} with different delays and quantities of expert demonstrations.}
    \label{fig:InvertedPendulum}
\end{figure}

\begin{figure}[h]
    \centering
    \centerline{
        \subfigure[Delay=5, $\#$Traj=10]{\includegraphics[width=0.33\linewidth]{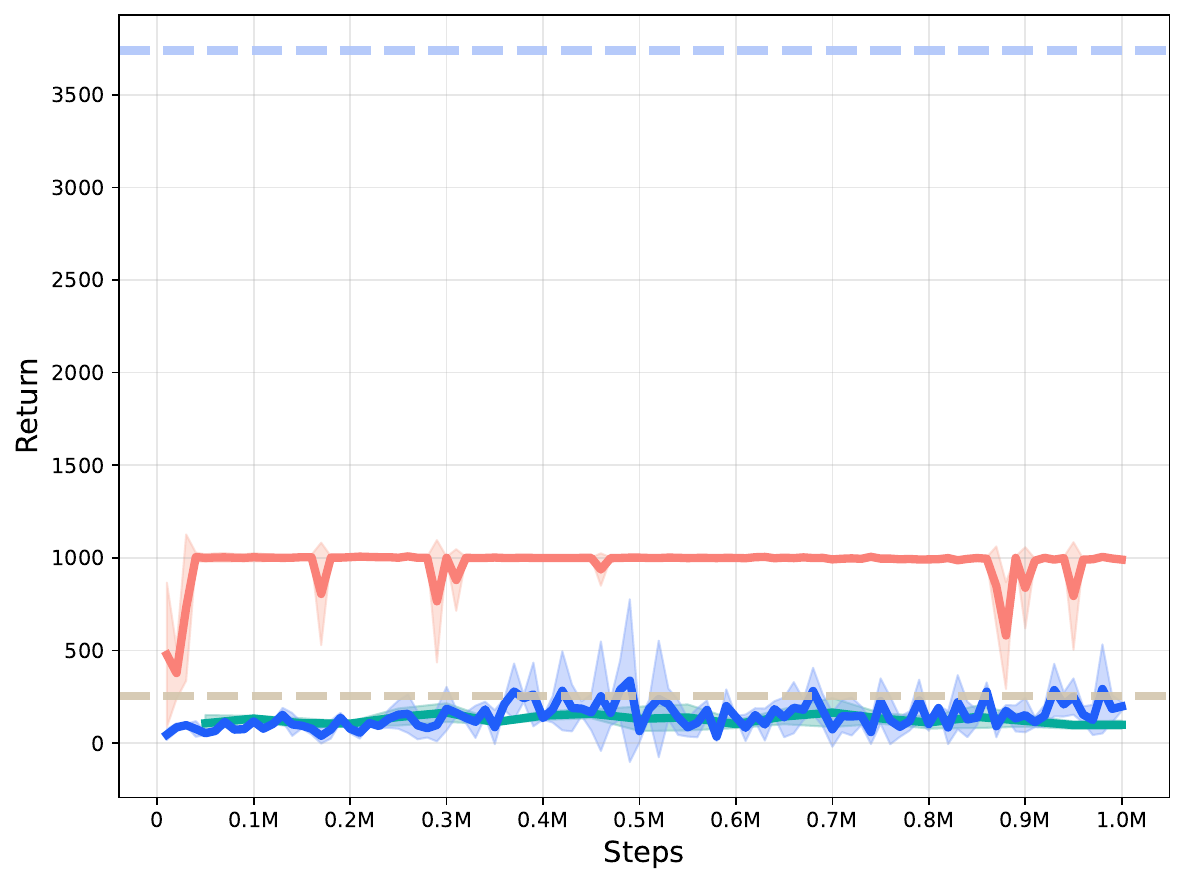}}
        \subfigure[Delay=5, $\#$Traj=100]{\includegraphics[width=0.33\linewidth]{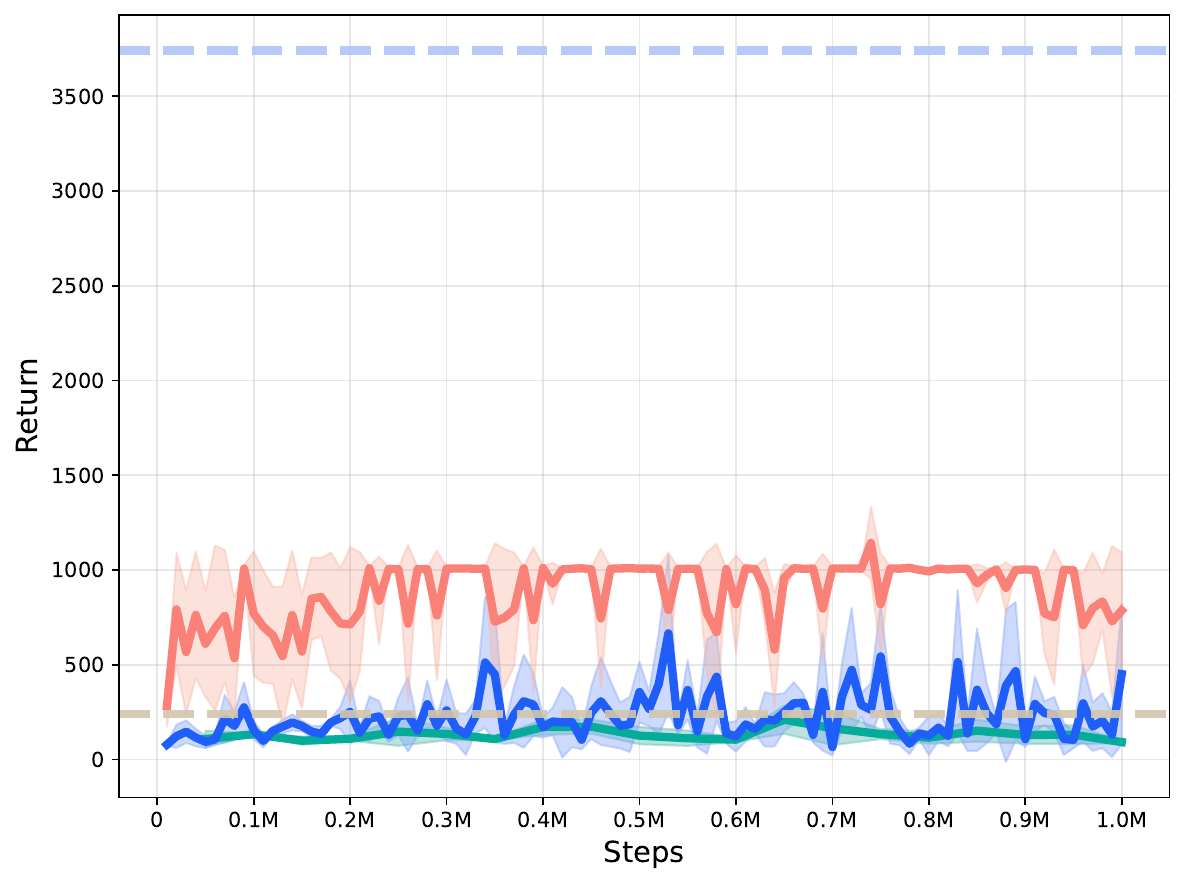}}
        \subfigure[Delay=5, $\#$Traj=1000]{\includegraphics[width=0.33\linewidth]{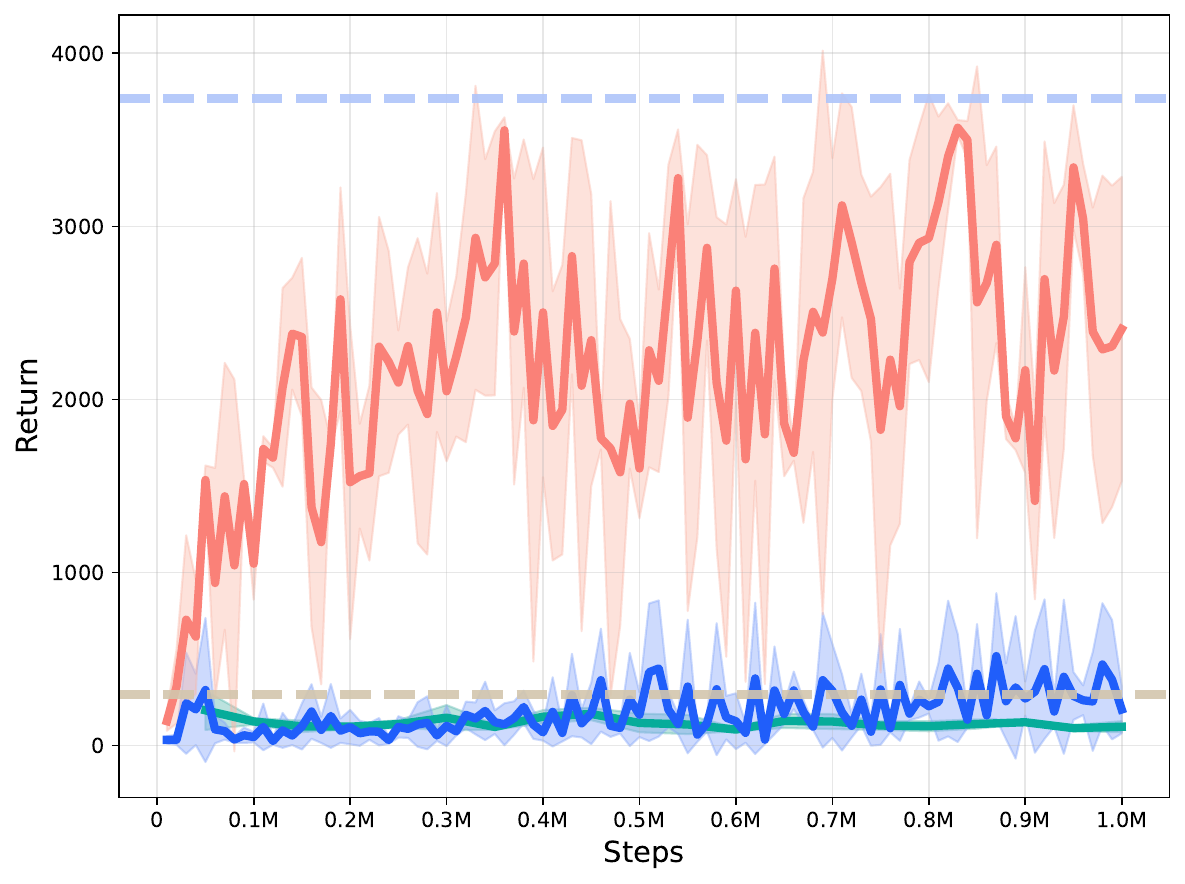}}
    }
    \centerline{
        \subfigure[Delay=10, $\#$Traj=10]{\includegraphics[width=0.33\linewidth]{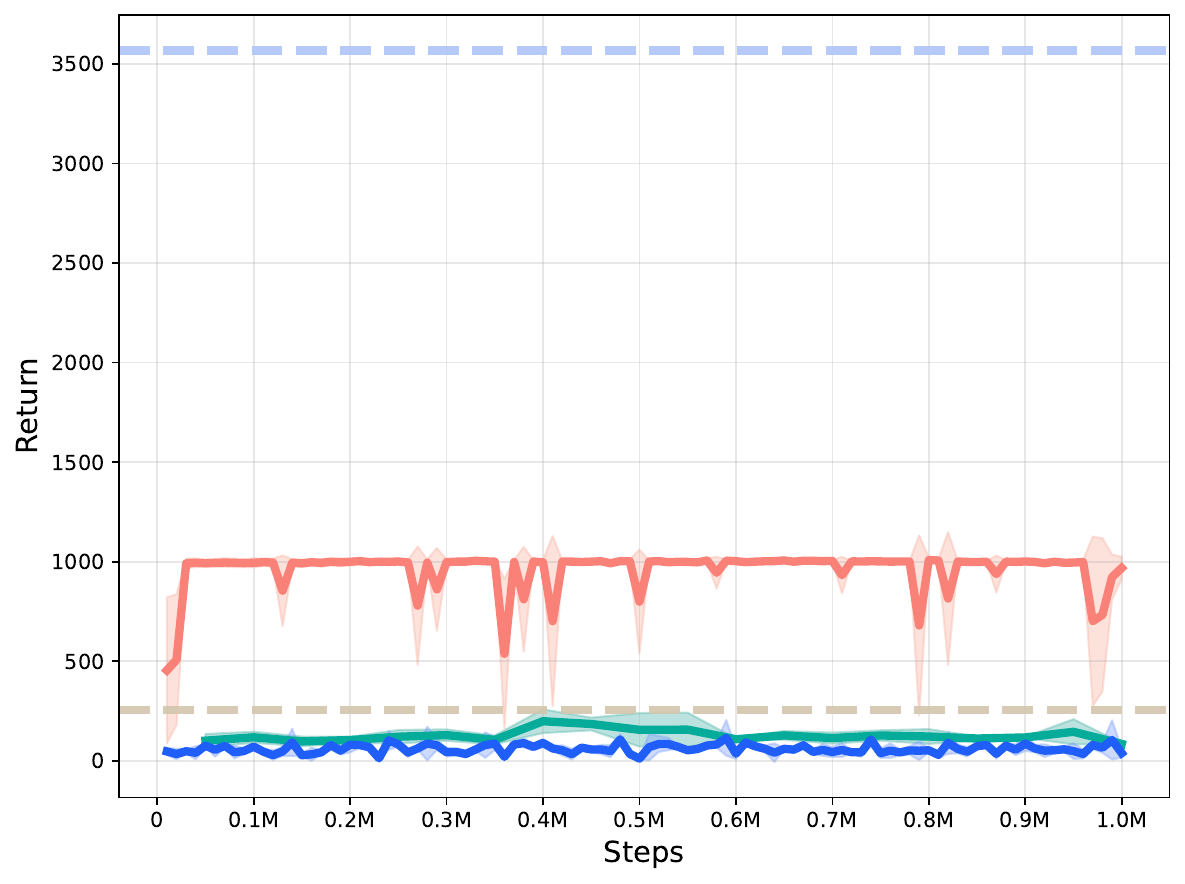}}
        \subfigure[Delay=10, $\#$Traj=100]{\includegraphics[width=0.33\linewidth]{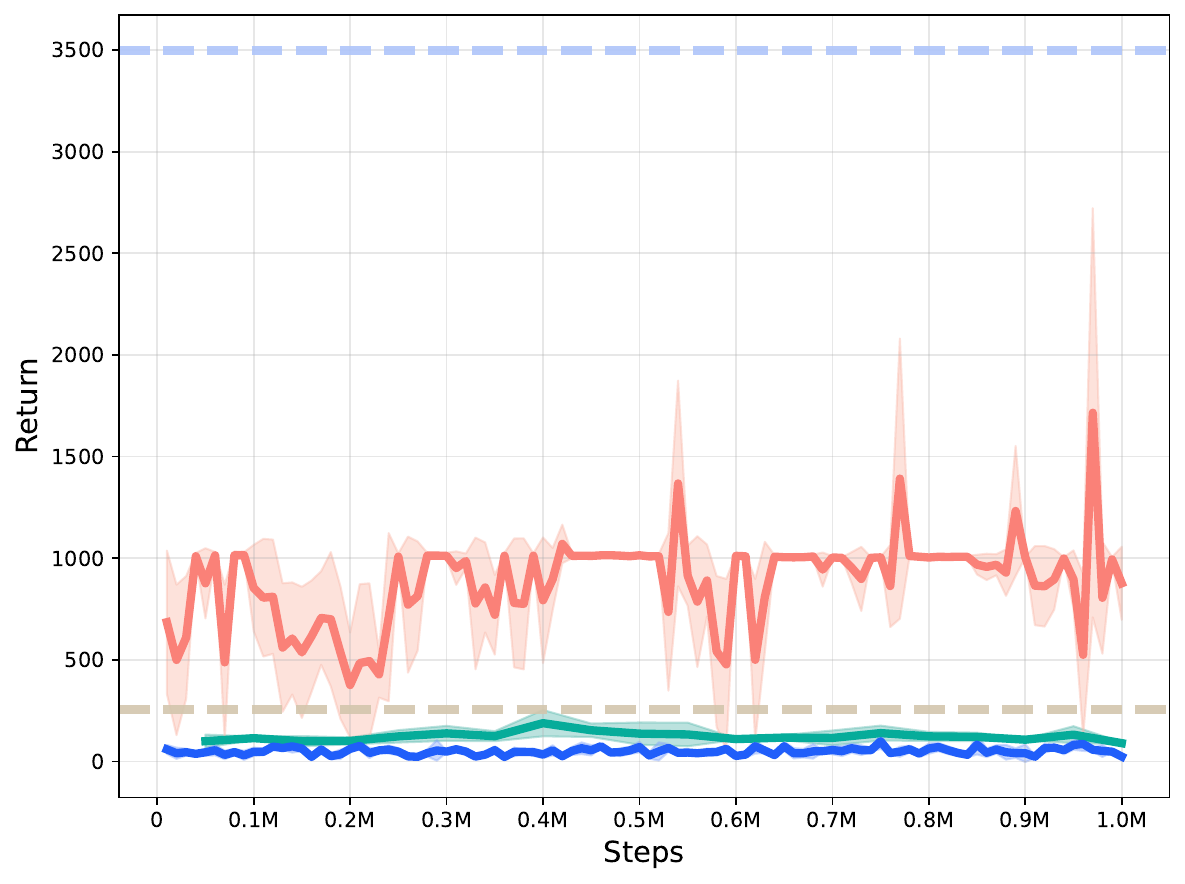}}
        \subfigure[Delay=10, $\#$Traj=1000]{\includegraphics[width=0.33\linewidth]{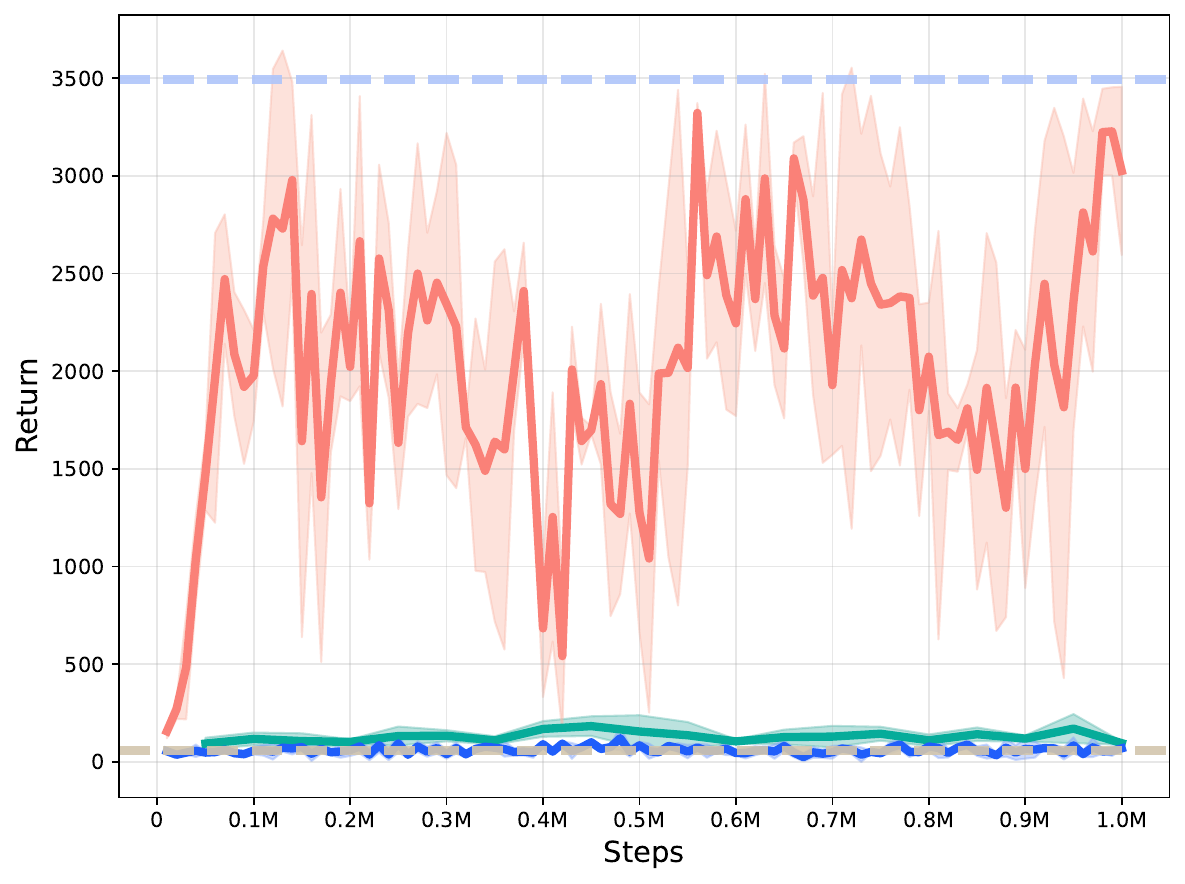}}
    }
    \centerline{
        \subfigure[Delay=25, $\#$Traj=10]{\includegraphics[width=0.33\linewidth]{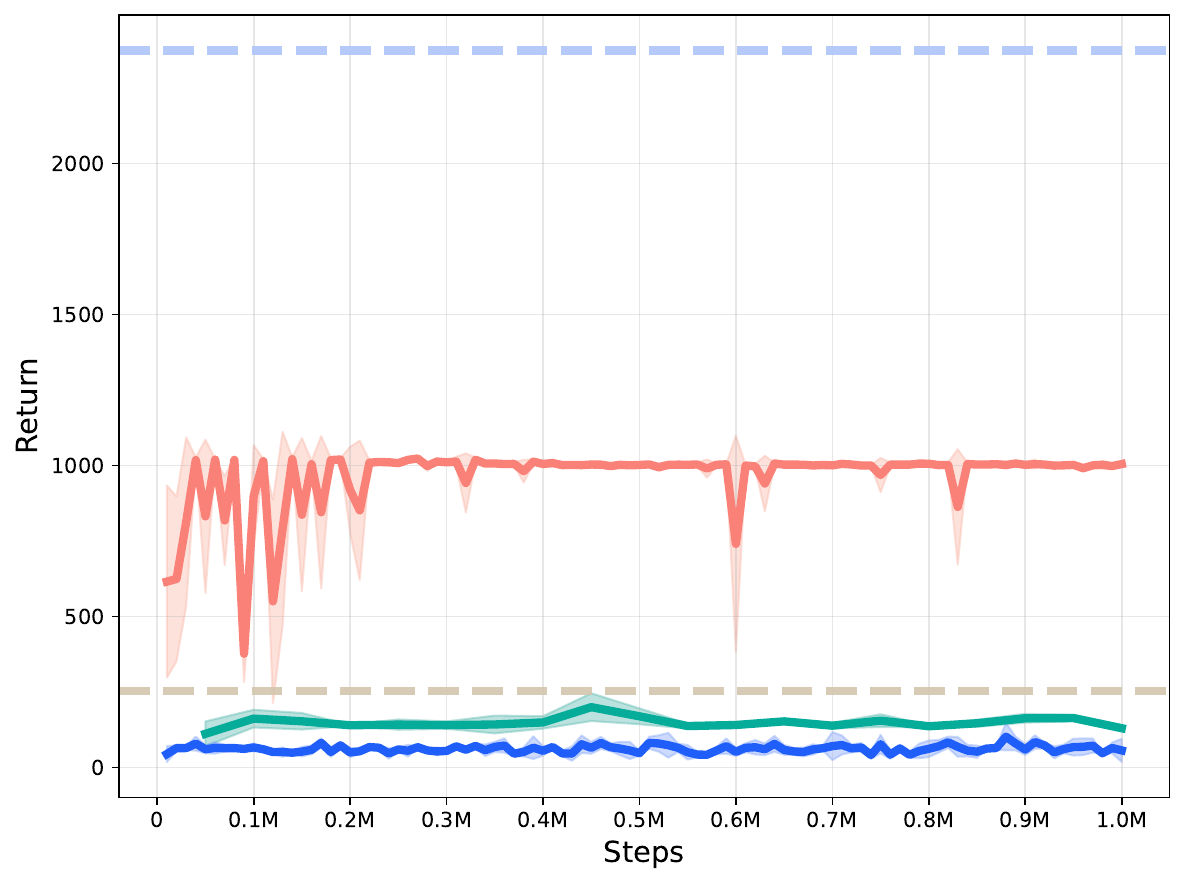}}
        \subfigure[Delay=25, $\#$Traj=100]{\includegraphics[width=0.33\linewidth]{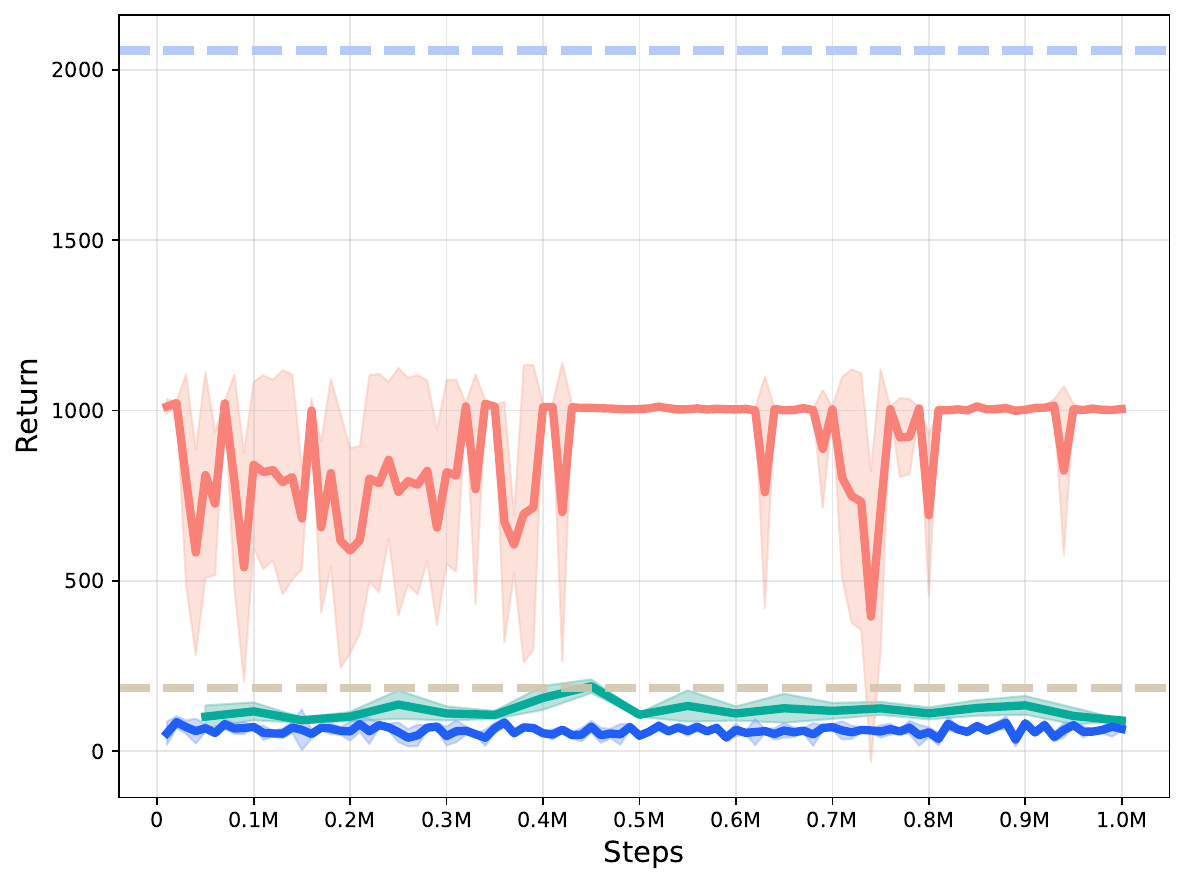}}
        \subfigure[Delay=25, $\#$Traj=1000]{\includegraphics[width=0.33\linewidth]{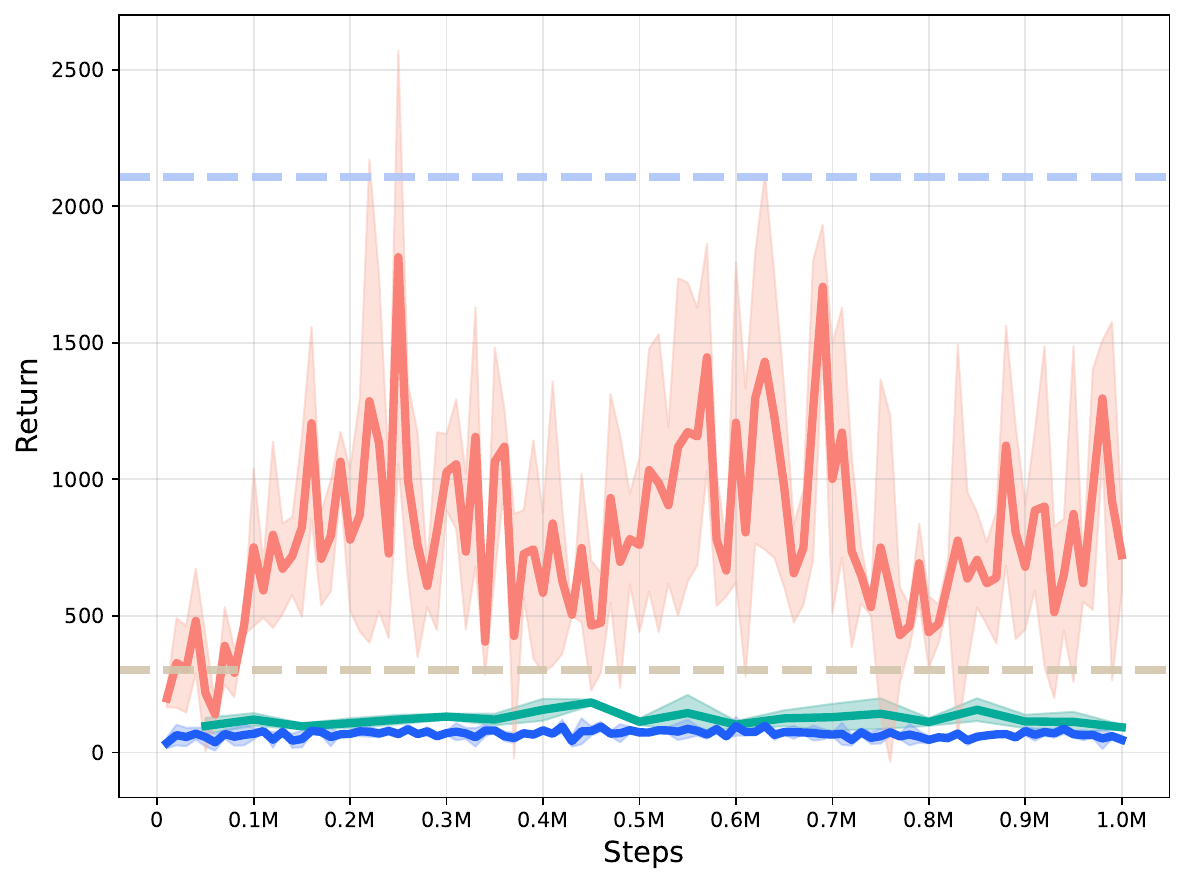}}
    }
    \centerline{
        \includegraphics[width=0.8\linewidth]{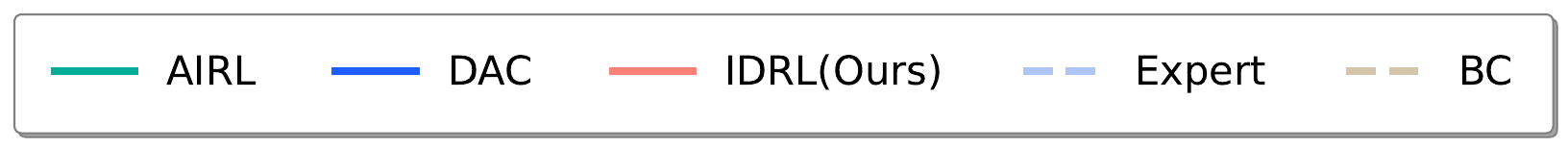}
    }
    \caption{Learning Curves on \texttt{Hopper-v4} with different delays and quantities of expert demonstrations.}
\end{figure}

\begin{figure}[h]
    \centering
    \centerline{
        \subfigure[Delay=5, $\#$Traj=10]{\includegraphics[width=0.33\linewidth]{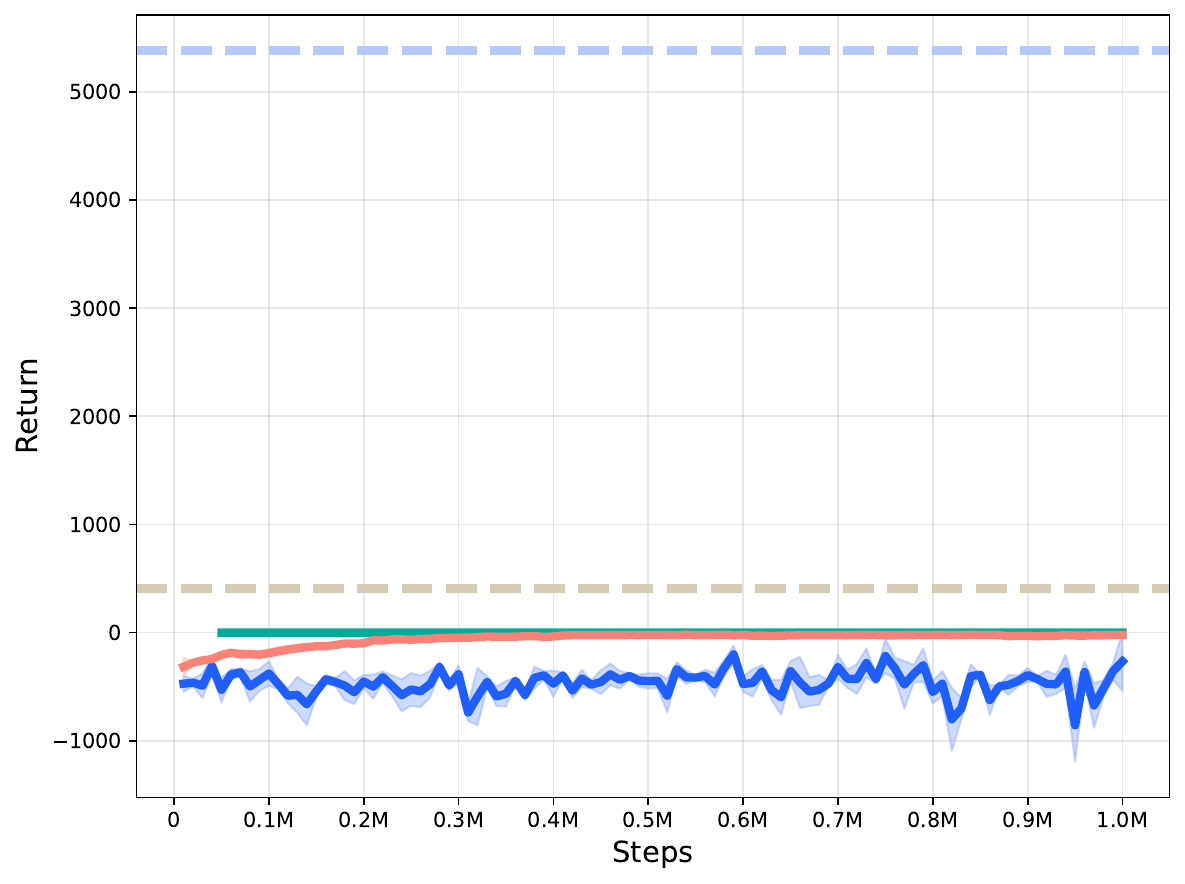}}
        \subfigure[Delay=5, $\#$Traj=100]{\includegraphics[width=0.33\linewidth]{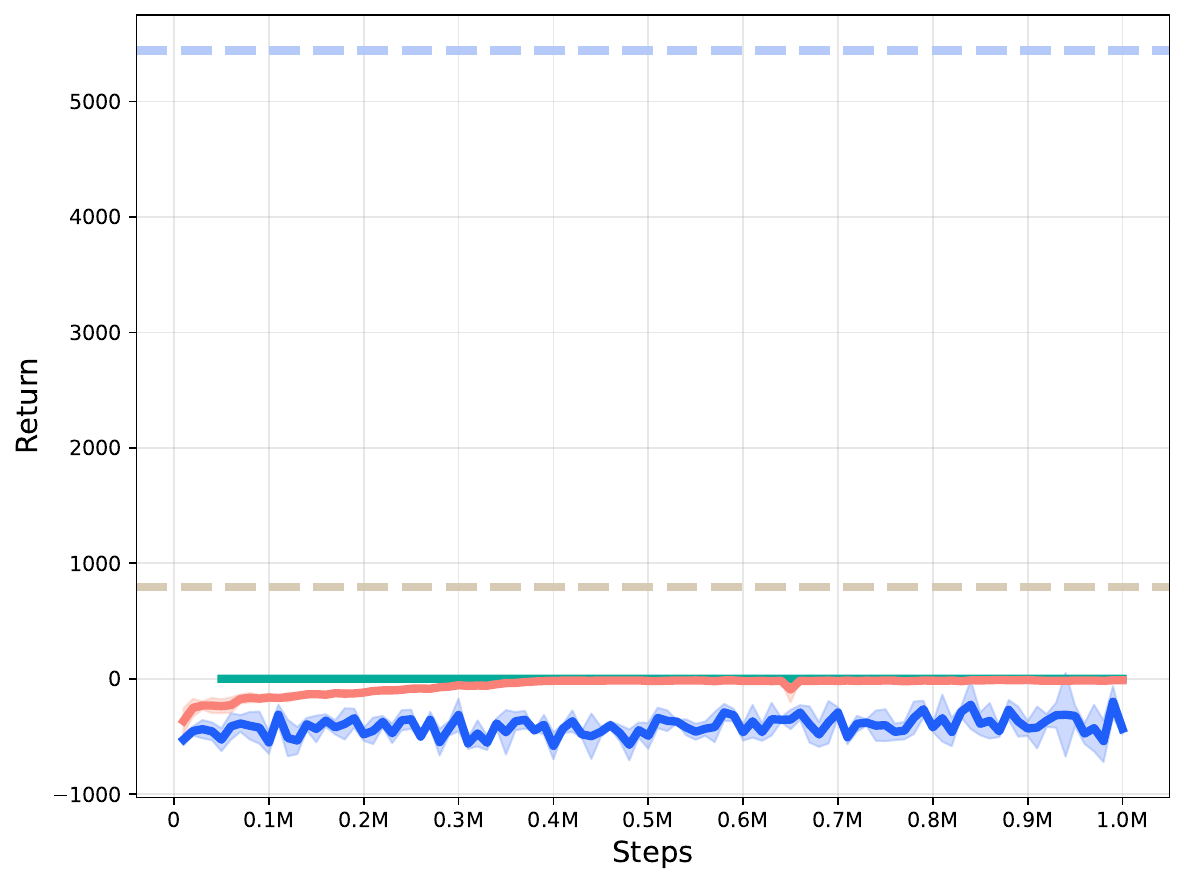}}
        \subfigure[Delay=5, $\#$Traj=1000]{\includegraphics[width=0.33\linewidth]{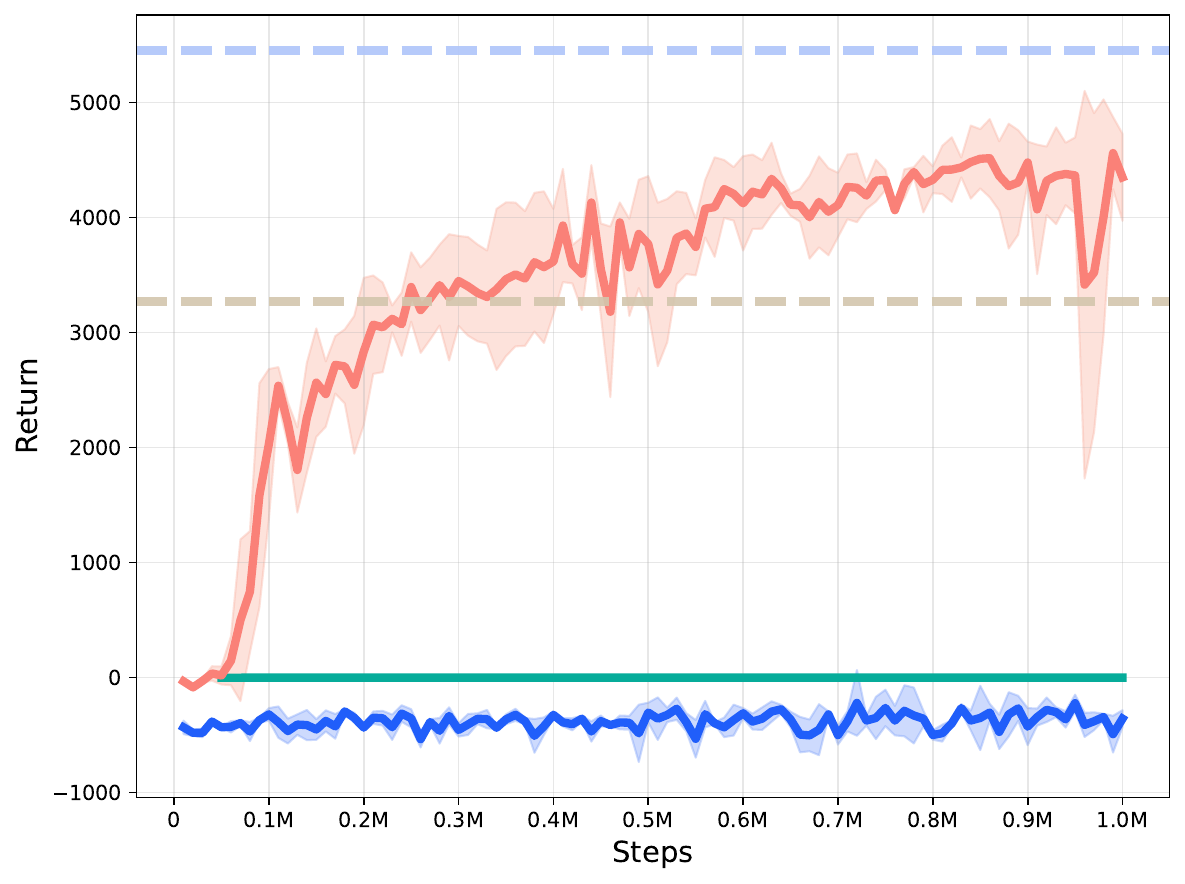}}
    }
    \centerline{
        \subfigure[Delay=10, $\#$Traj=10]{\includegraphics[width=0.33\linewidth]{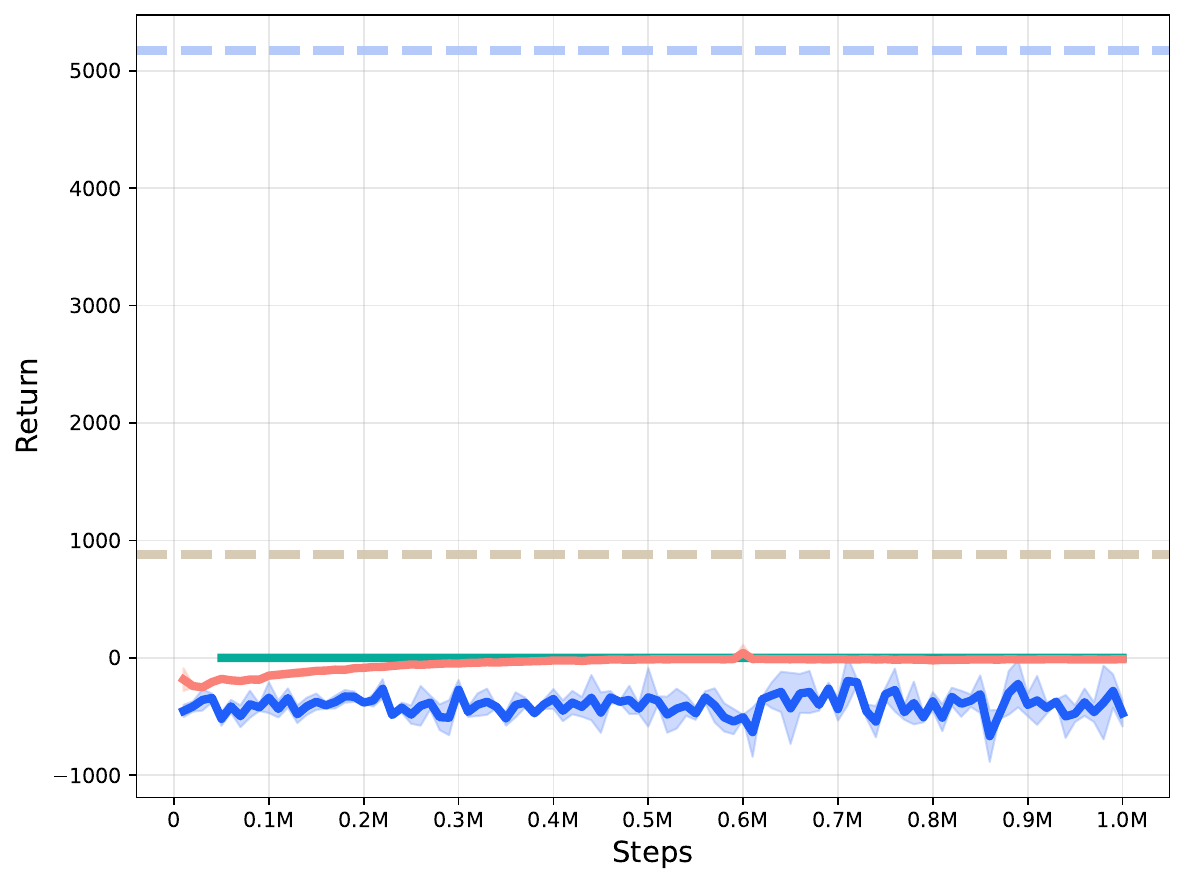}}
        \subfigure[Delay=10, $\#$Traj=100]{\includegraphics[width=0.33\linewidth]{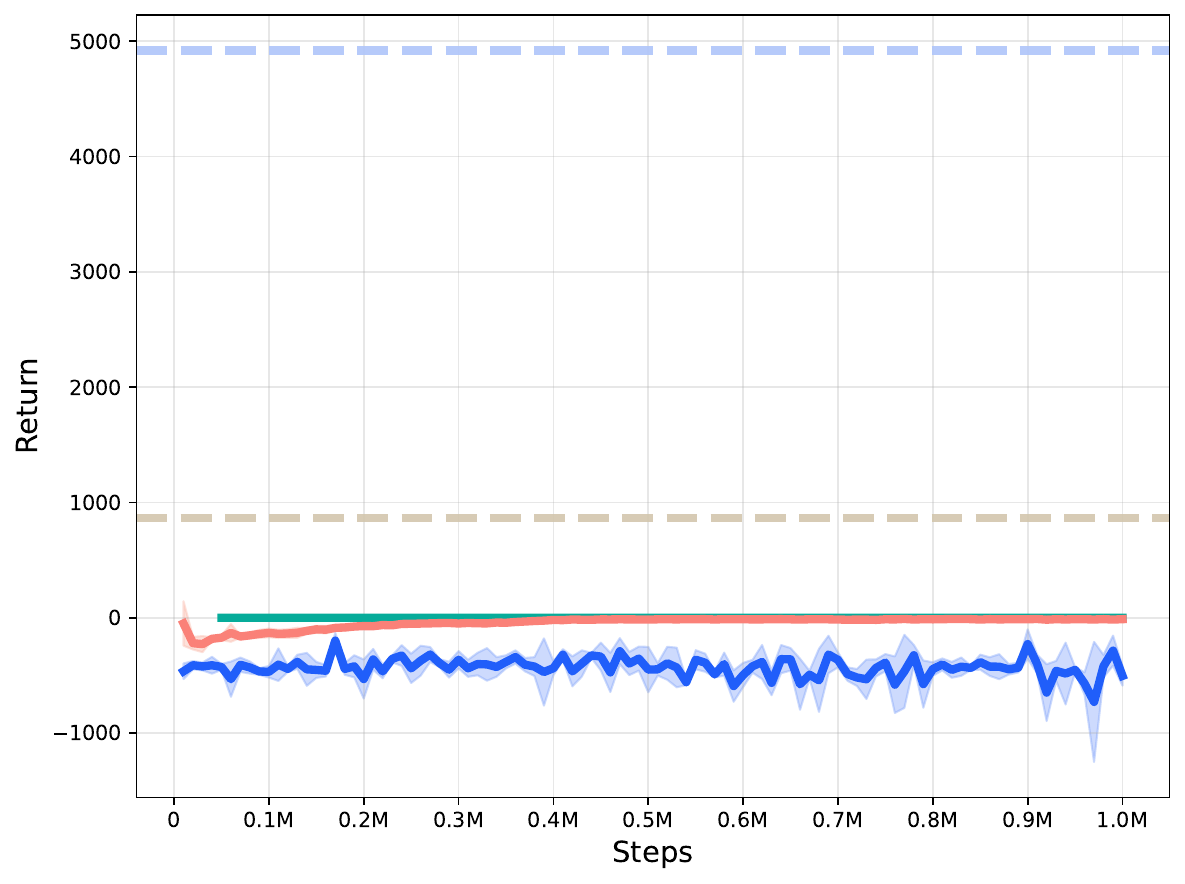}}
        \subfigure[Delay=10, $\#$Traj=1000]{\includegraphics[width=0.33\linewidth]{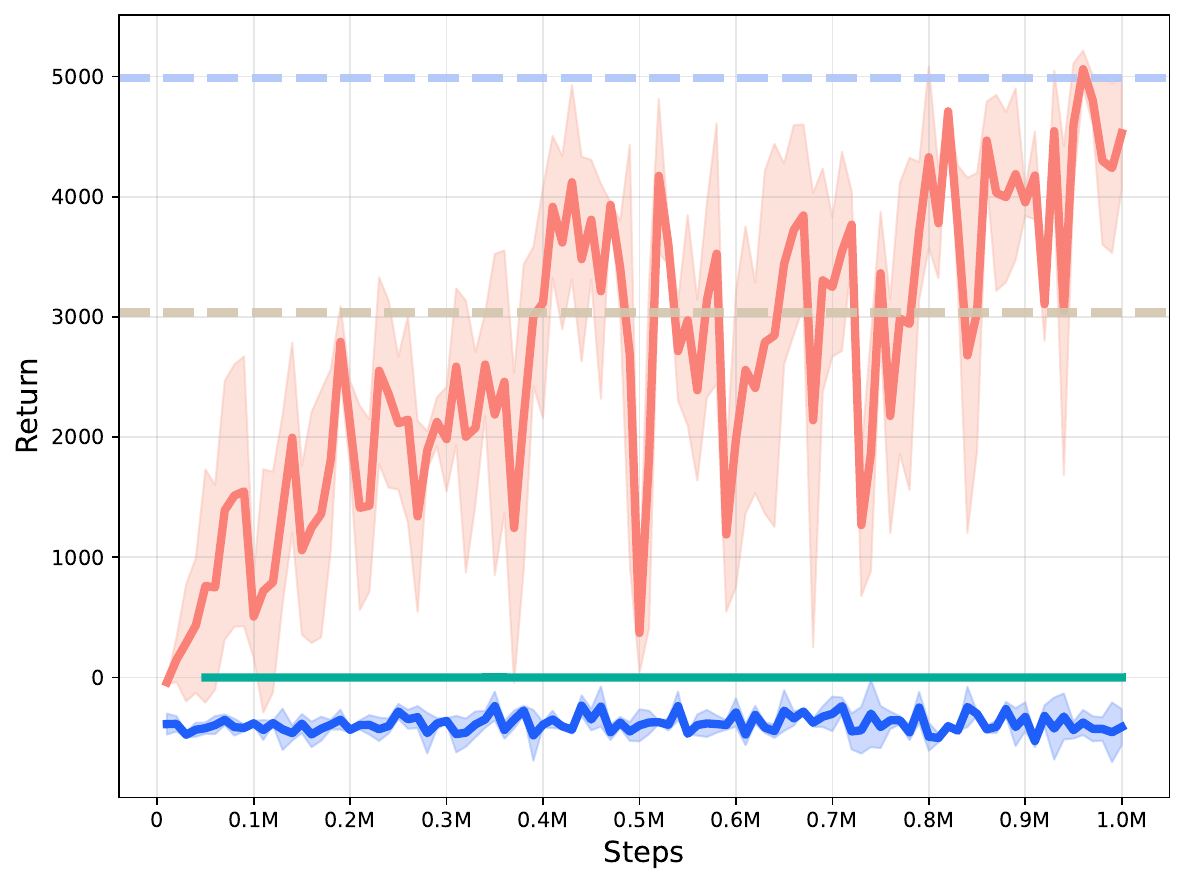}}
    }
    \centerline{
        \subfigure[Delay=25, $\#$Traj=10]{\includegraphics[width=0.33\linewidth]{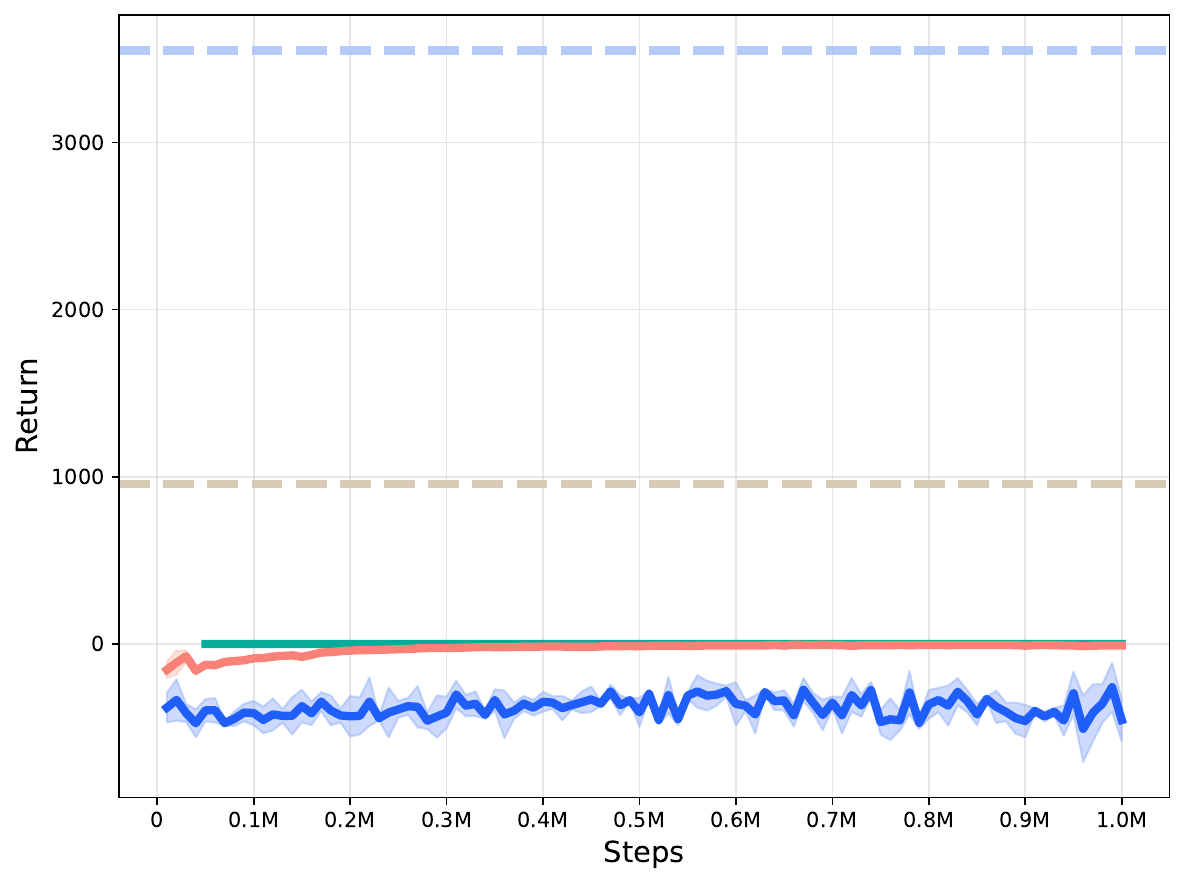}}
        \subfigure[Delay=25, $\#$Traj=100]{\includegraphics[width=0.33\linewidth]{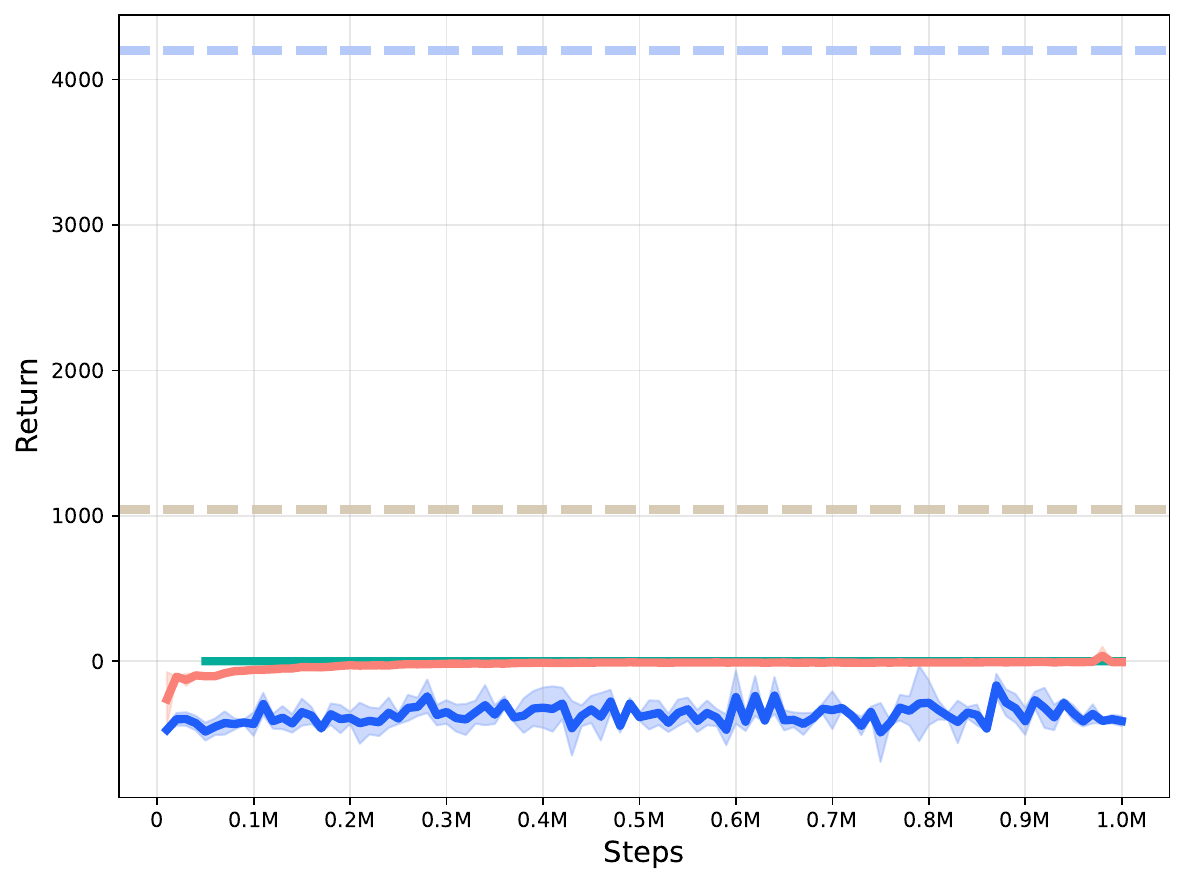}}
        \subfigure[Delay=25, $\#$Traj=1000]{\includegraphics[width=0.33\linewidth]{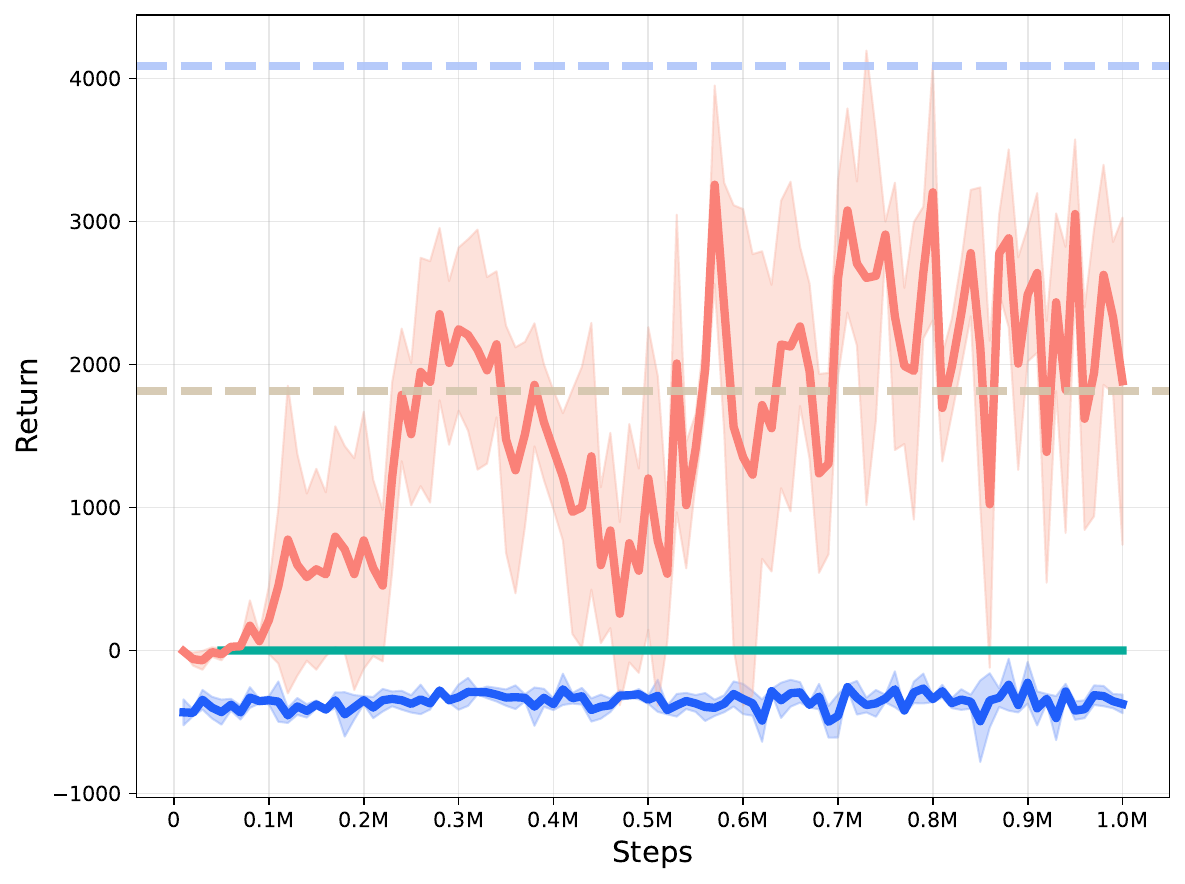}}
    }
    \centerline{
        \includegraphics[width=0.8\linewidth]{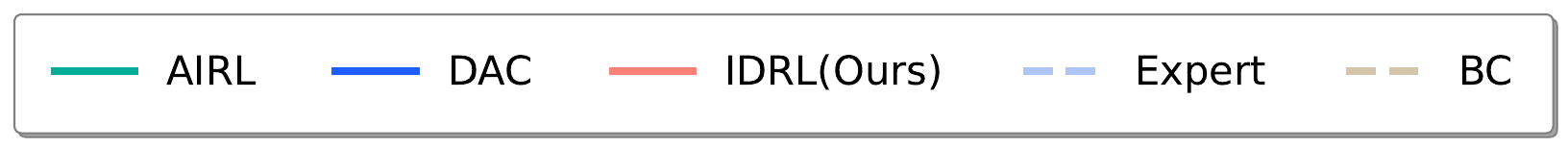}
    }
    \caption{Learning Curves on \texttt{HalfCheetah-v4} with different delays and quantities of expert demonstrations.}
\end{figure}

\begin{figure}[h]
    \centering
    \centerline{
        \subfigure[Delay=5, $\#$Traj=10]{\includegraphics[width=0.33\linewidth]{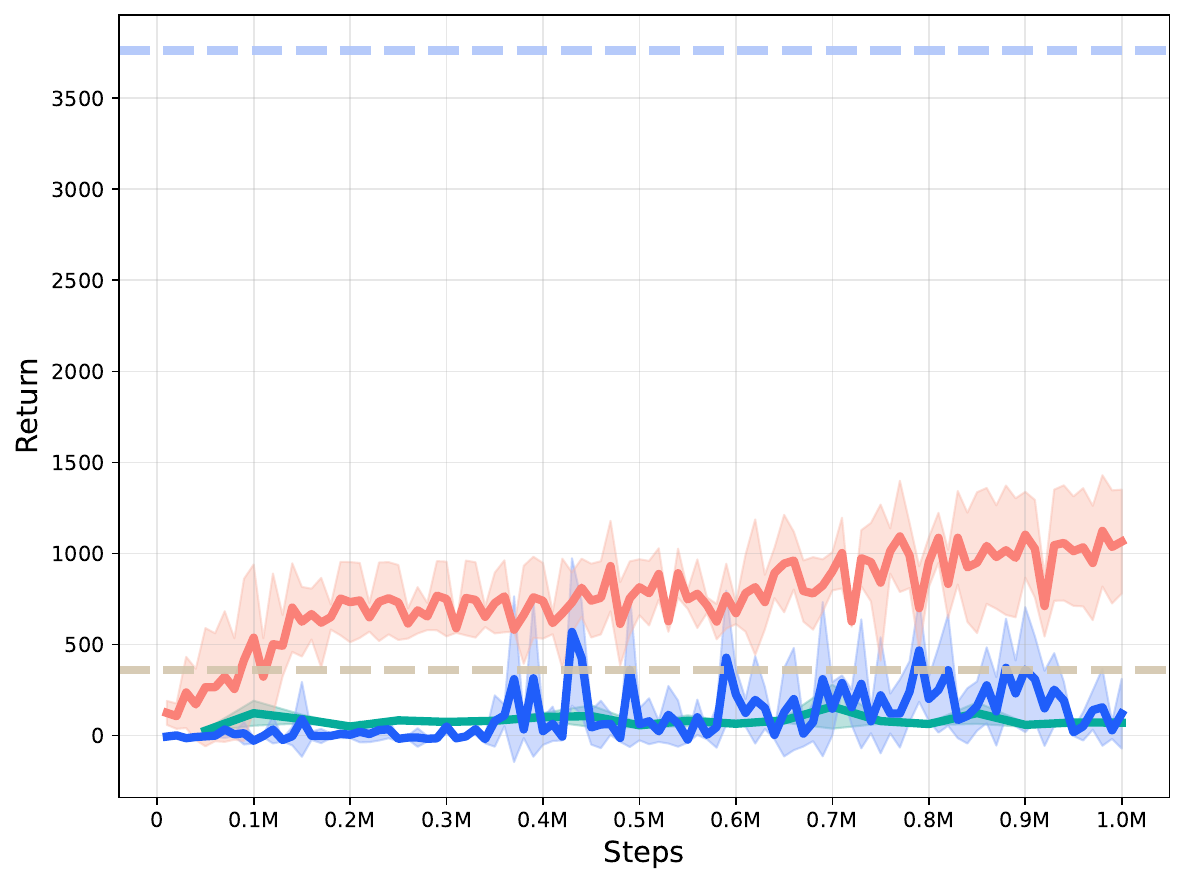}}
        \subfigure[Delay=5, $\#$Traj=100]{\includegraphics[width=0.33\linewidth]{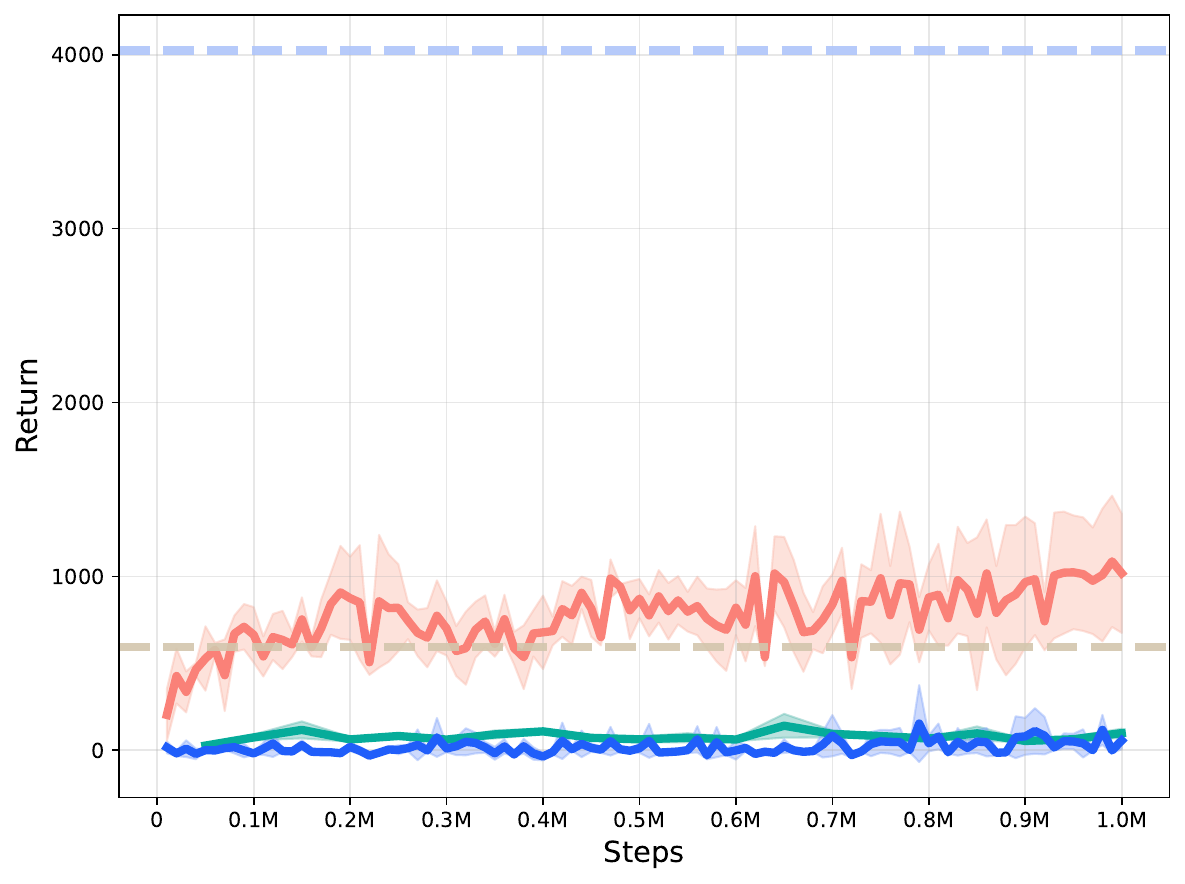}}
        \subfigure[Delay=5, $\#$Traj=1000]{\includegraphics[width=0.33\linewidth]{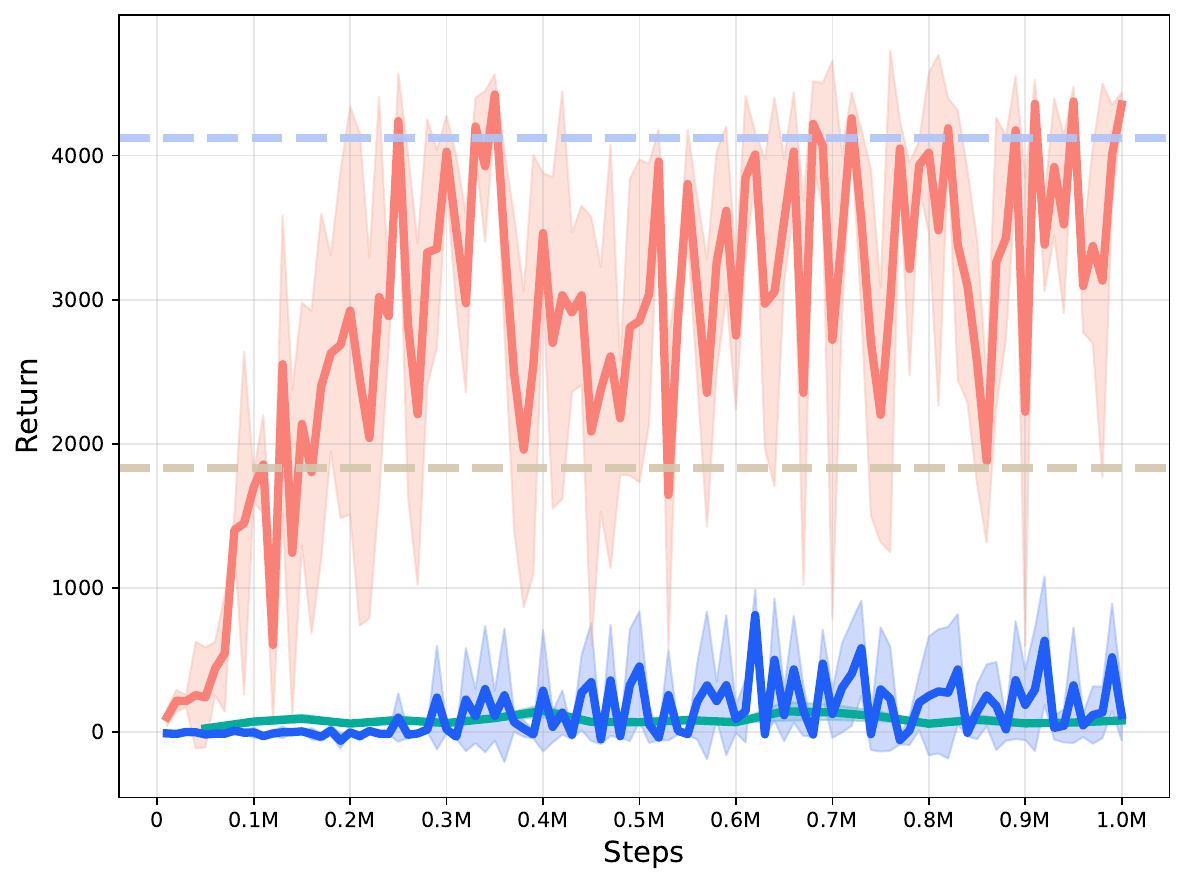}}
    }
    \centerline{
        \subfigure[Delay=10, $\#$Traj=10]{\includegraphics[width=0.33\linewidth]{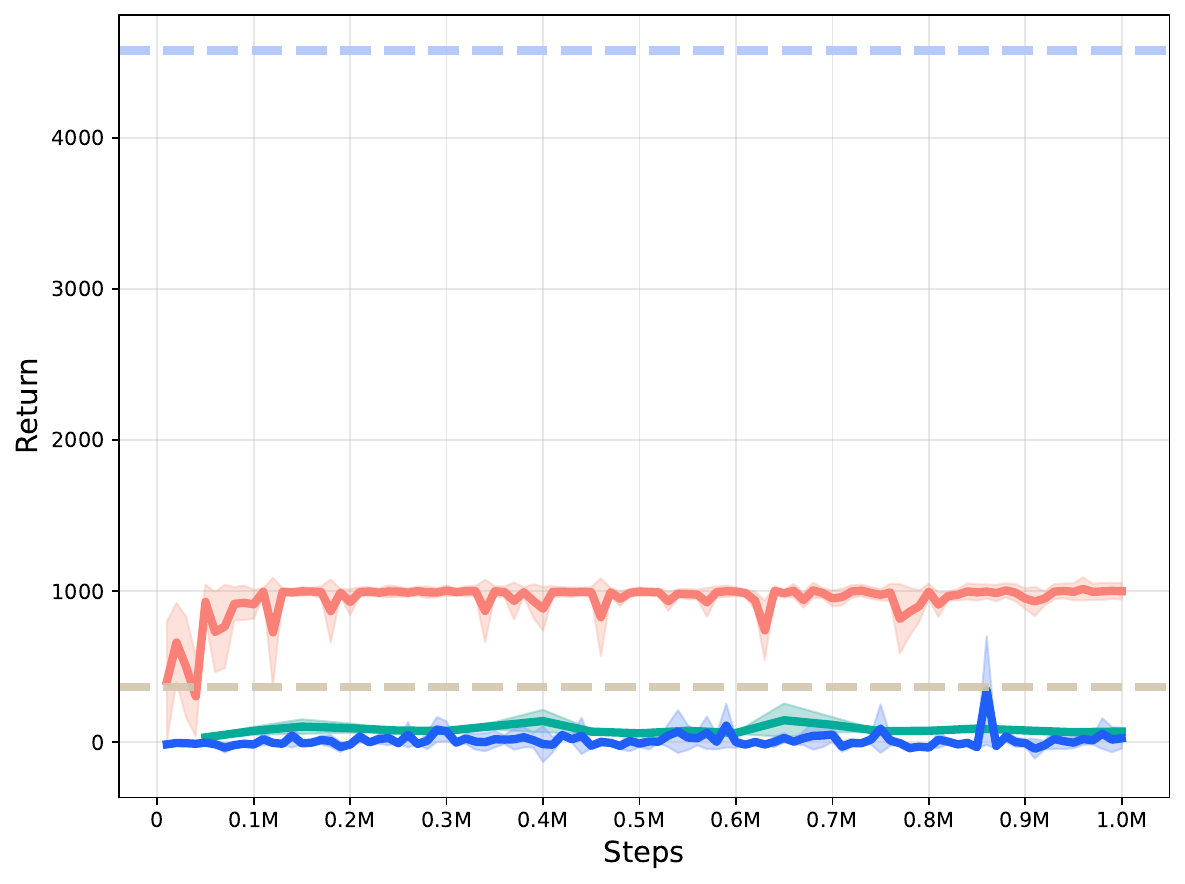}}
        \subfigure[Delay=10, $\#$Traj=100]{\includegraphics[width=0.33\linewidth]{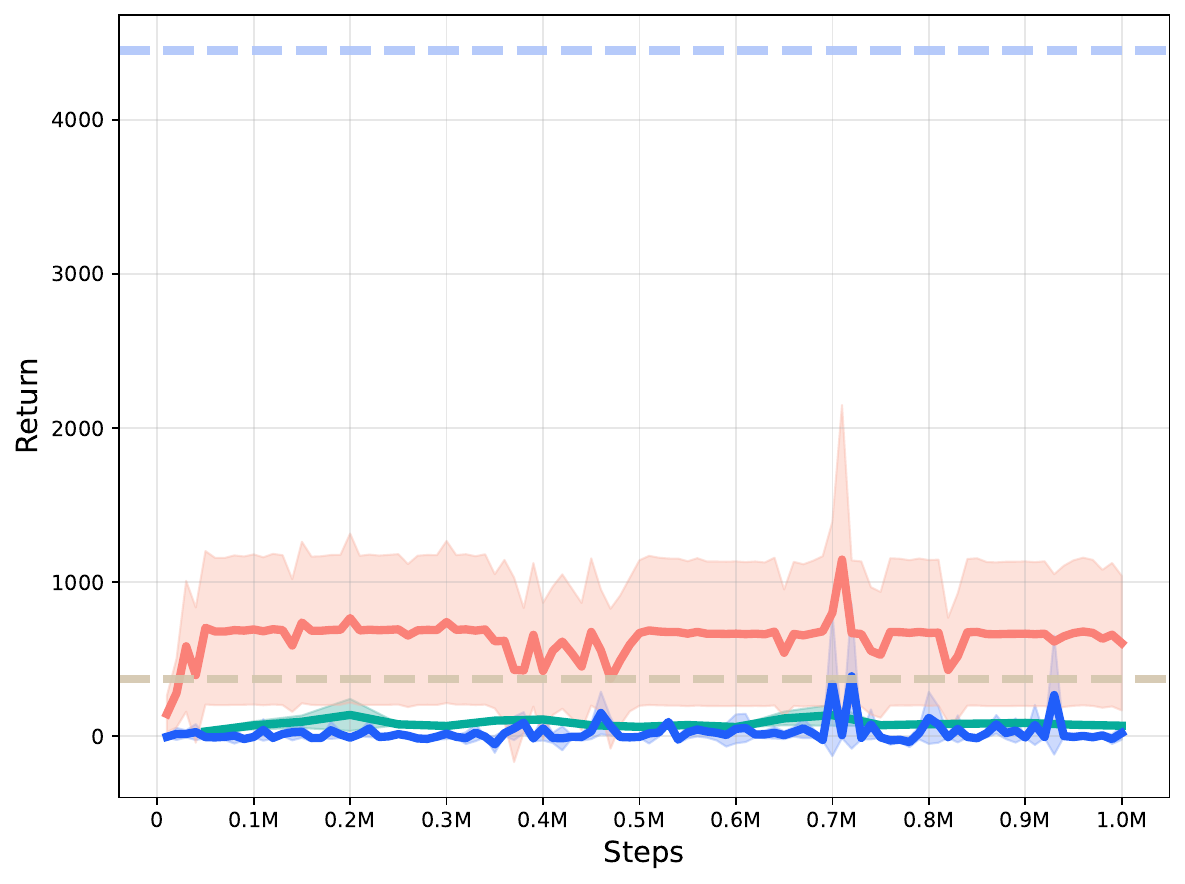}}
        \subfigure[Delay=10, $\#$Traj=1000]{\includegraphics[width=0.33\linewidth]{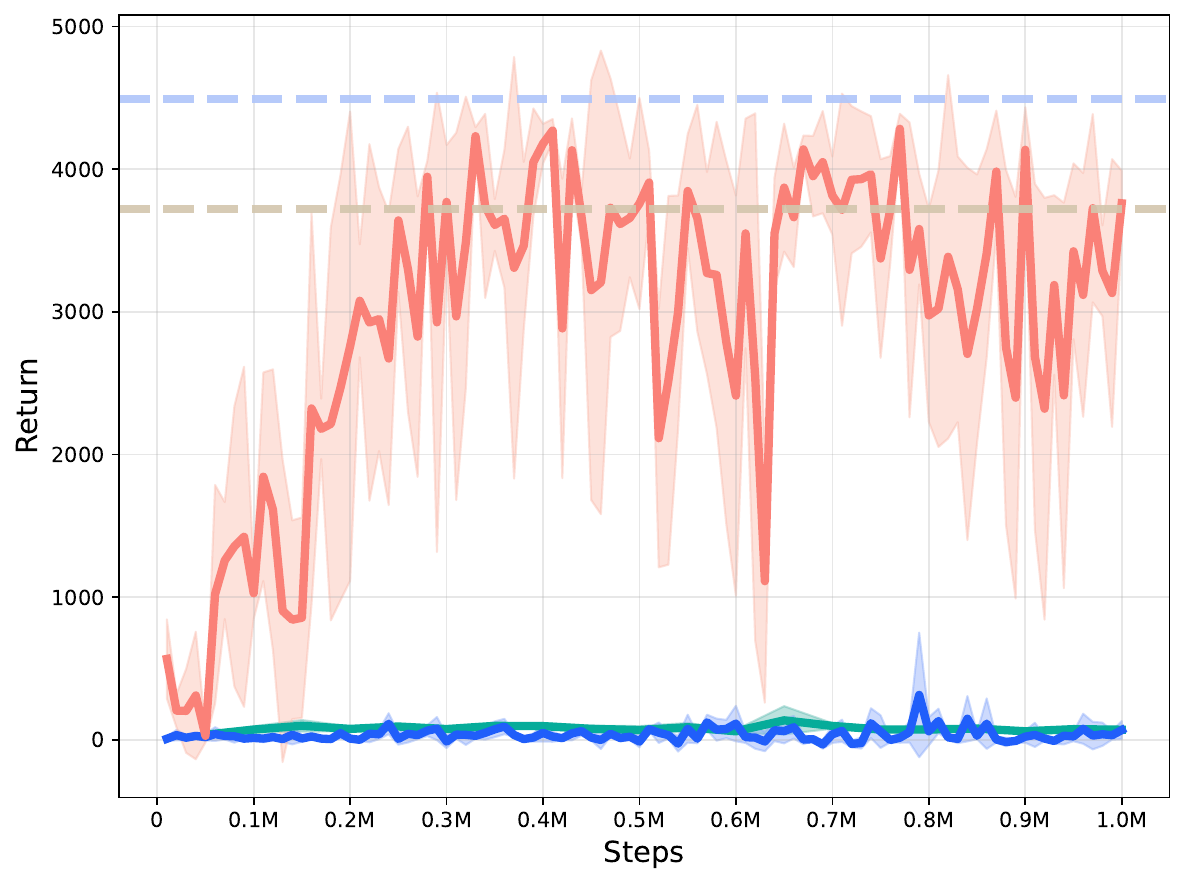}}
    }
    \centerline{
        \subfigure[Delay=25, $\#$Traj=10]{\includegraphics[width=0.33\linewidth]{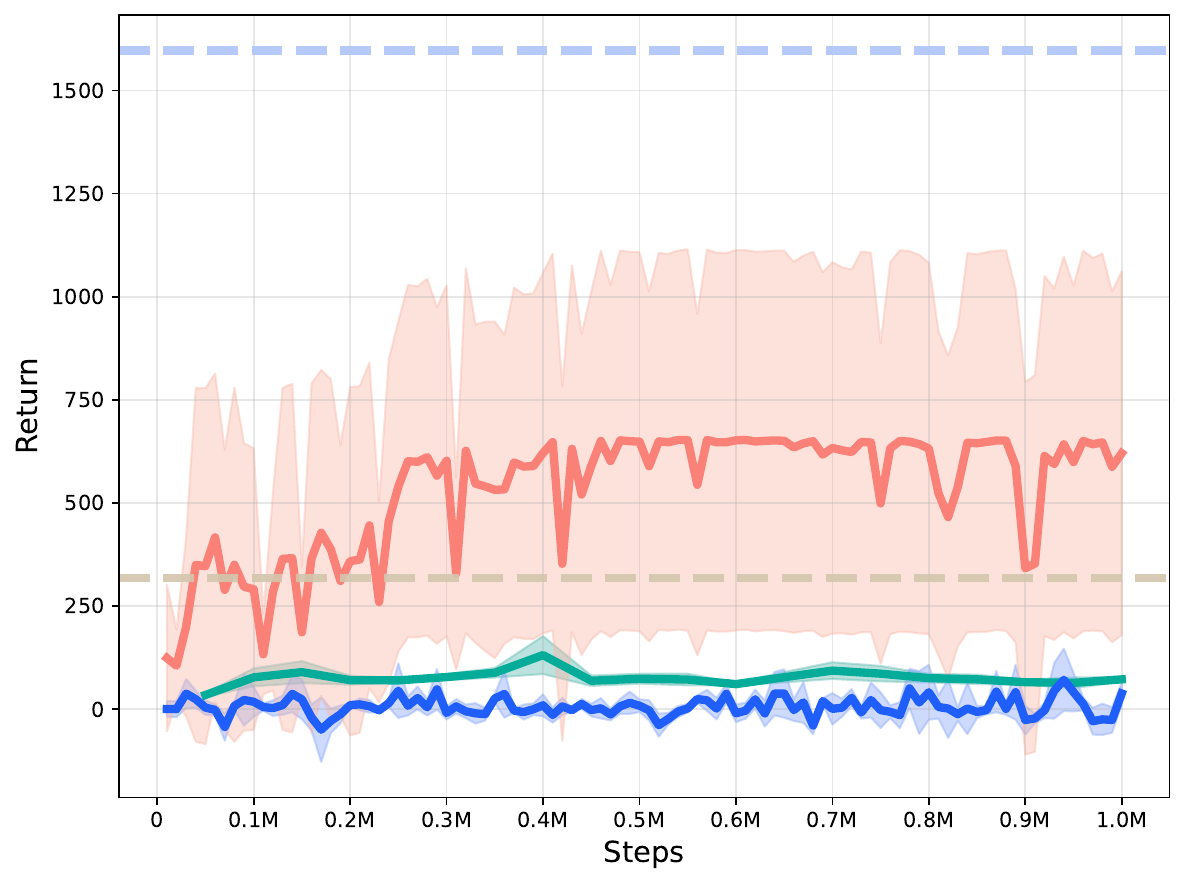}}
        \subfigure[Delay=25, $\#$Traj=100]{\includegraphics[width=0.33\linewidth]{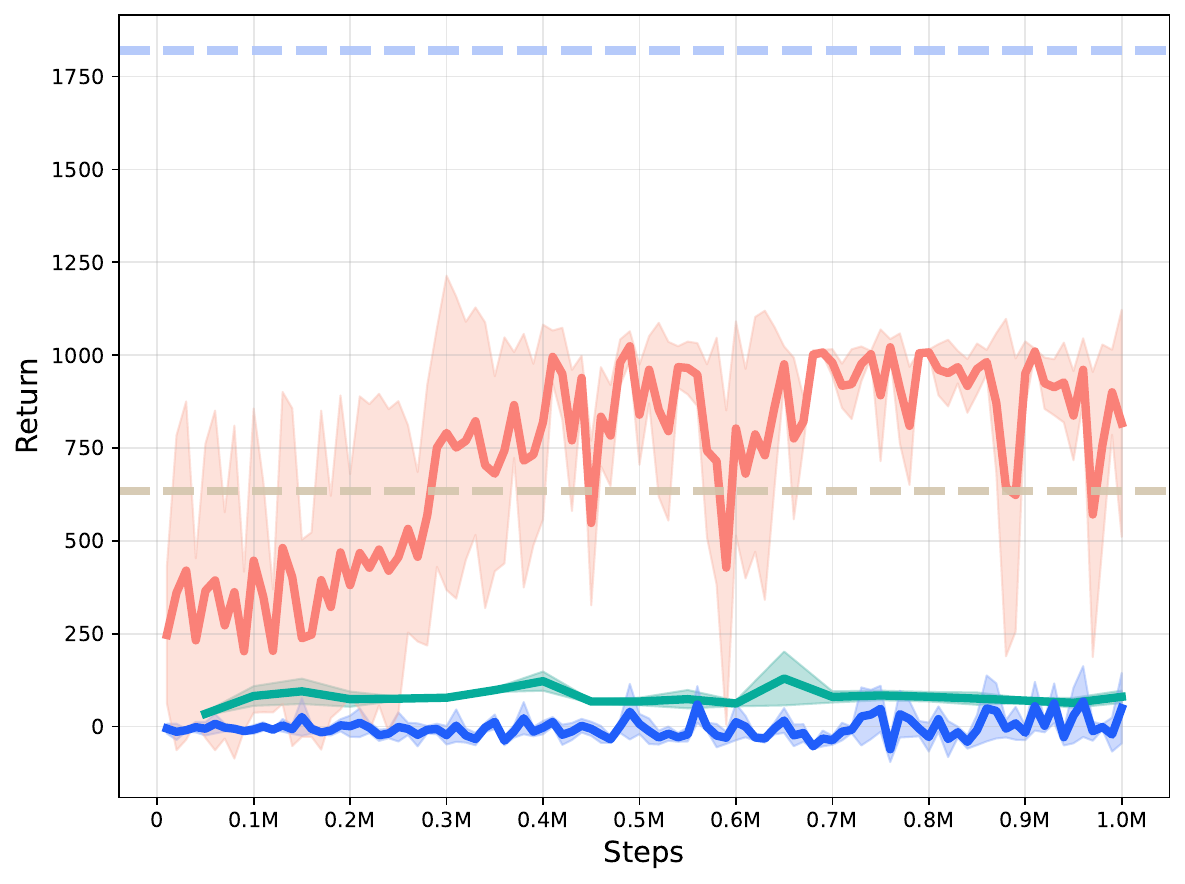}}
        \subfigure[Delay=25, $\#$Traj=1000]{\includegraphics[width=0.33\linewidth]{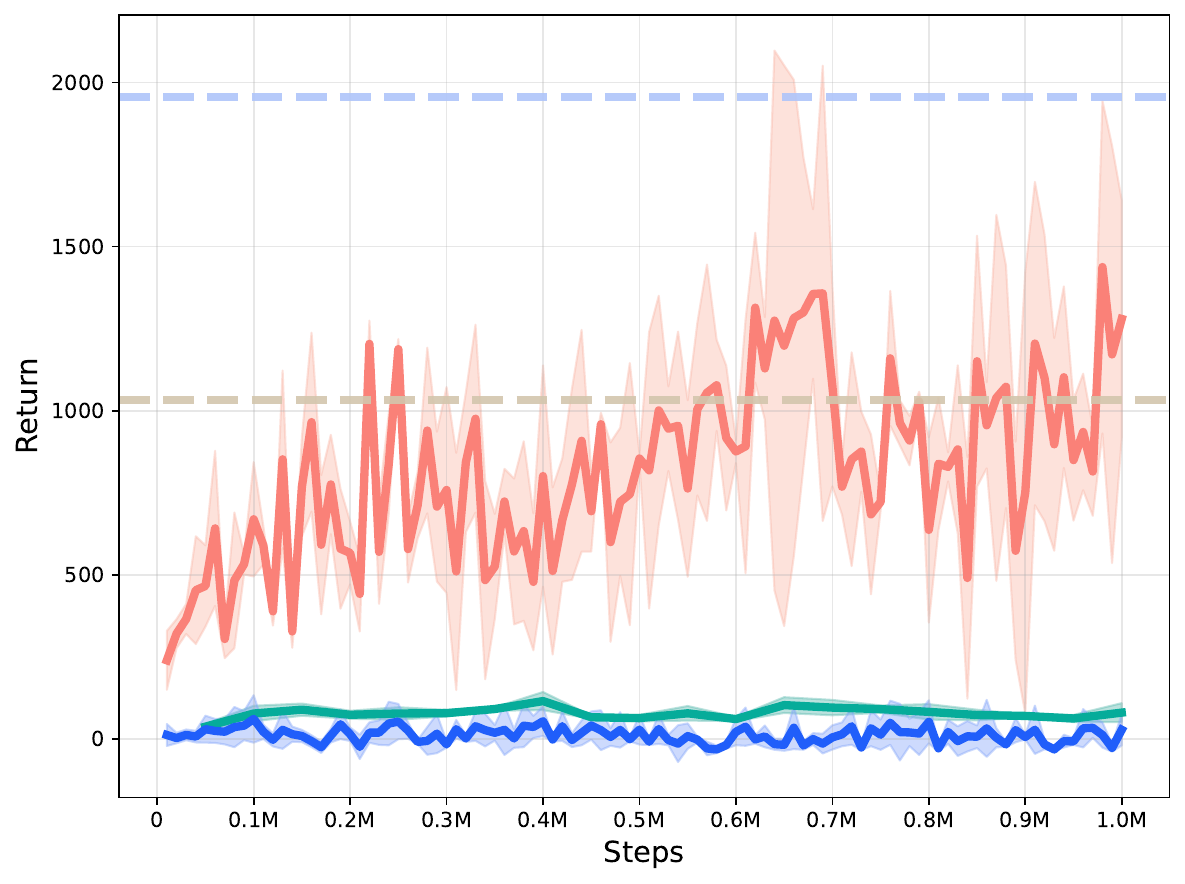}}
    }
    \centerline{
        \includegraphics[width=0.8\linewidth]{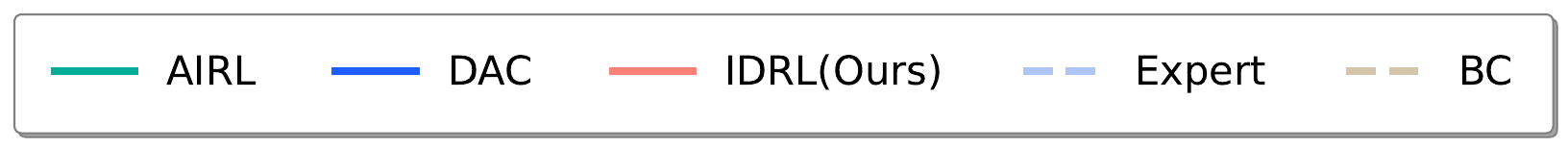}
    }
    \caption{Learning Curves on \texttt{Walker2d-v4} with different delays and quantities of expert demonstrations.}
\end{figure}

\begin{figure}[h]
    \centering
    \centerline{
        \subfigure[Delay=5, $\#$Traj=10]{\includegraphics[width=0.33\linewidth]{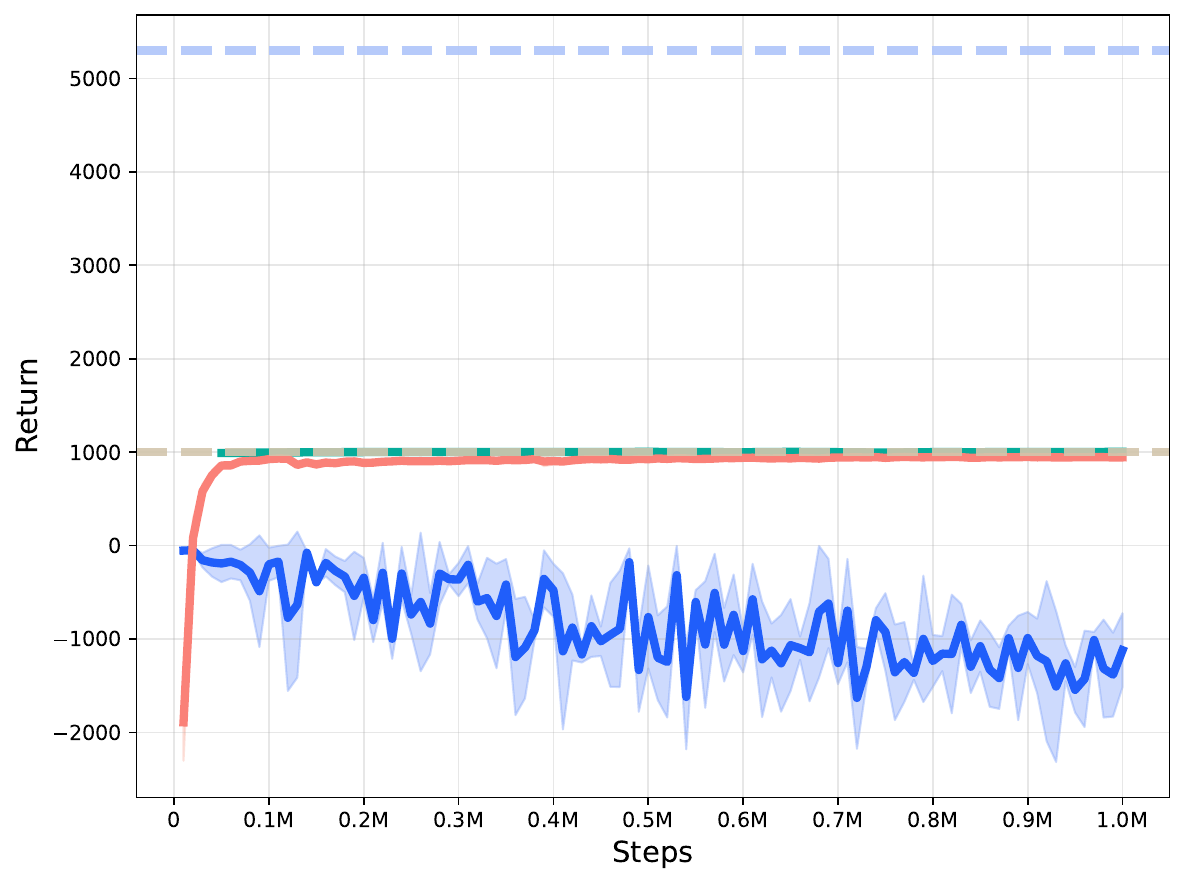}}
        \subfigure[Delay=5, $\#$Traj=100]{\includegraphics[width=0.33\linewidth]{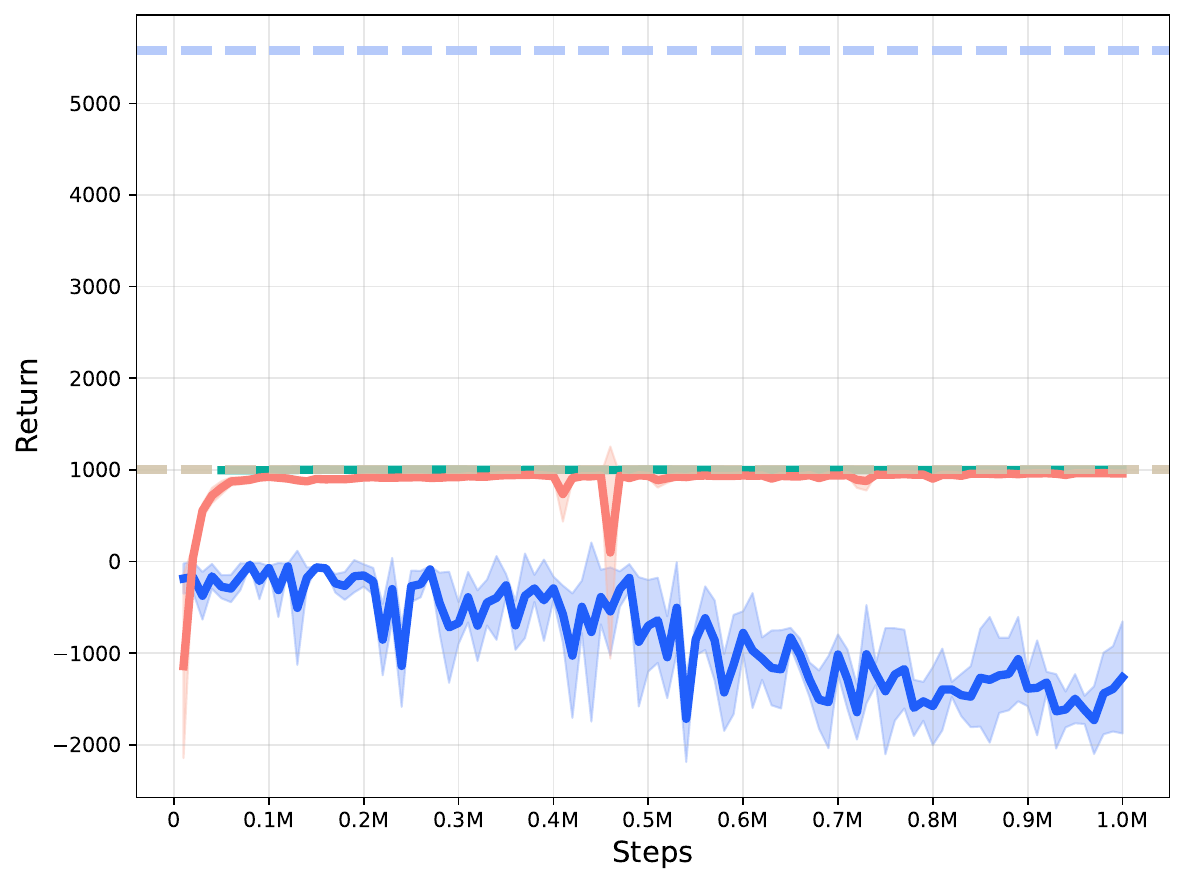}}
        \subfigure[Delay=5, $\#$Traj=1000]{\includegraphics[width=0.33\linewidth]{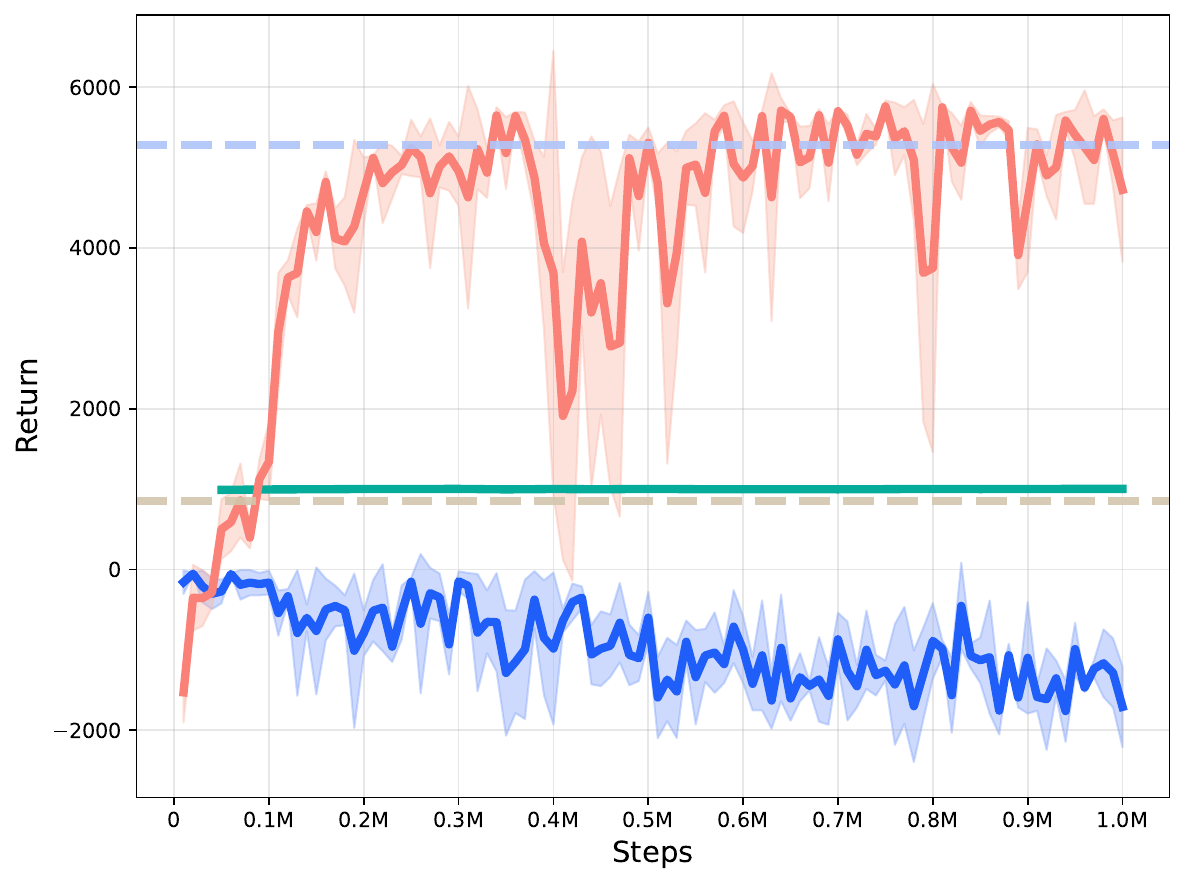}}
    }
    \centerline{
        \subfigure[Delay=10, $\#$Traj=10]{\includegraphics[width=0.33\linewidth]{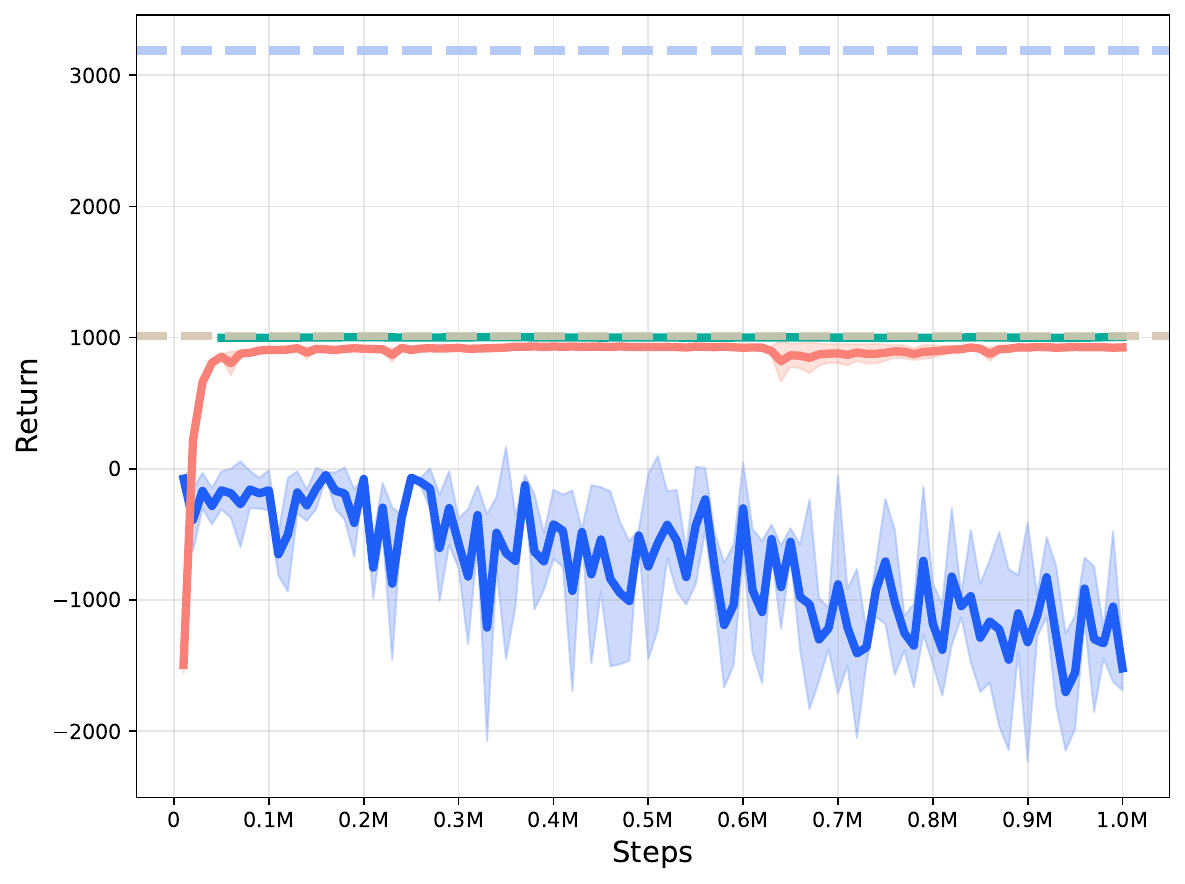}}
        \subfigure[Delay=10, $\#$Traj=100]{\includegraphics[width=0.33\linewidth]{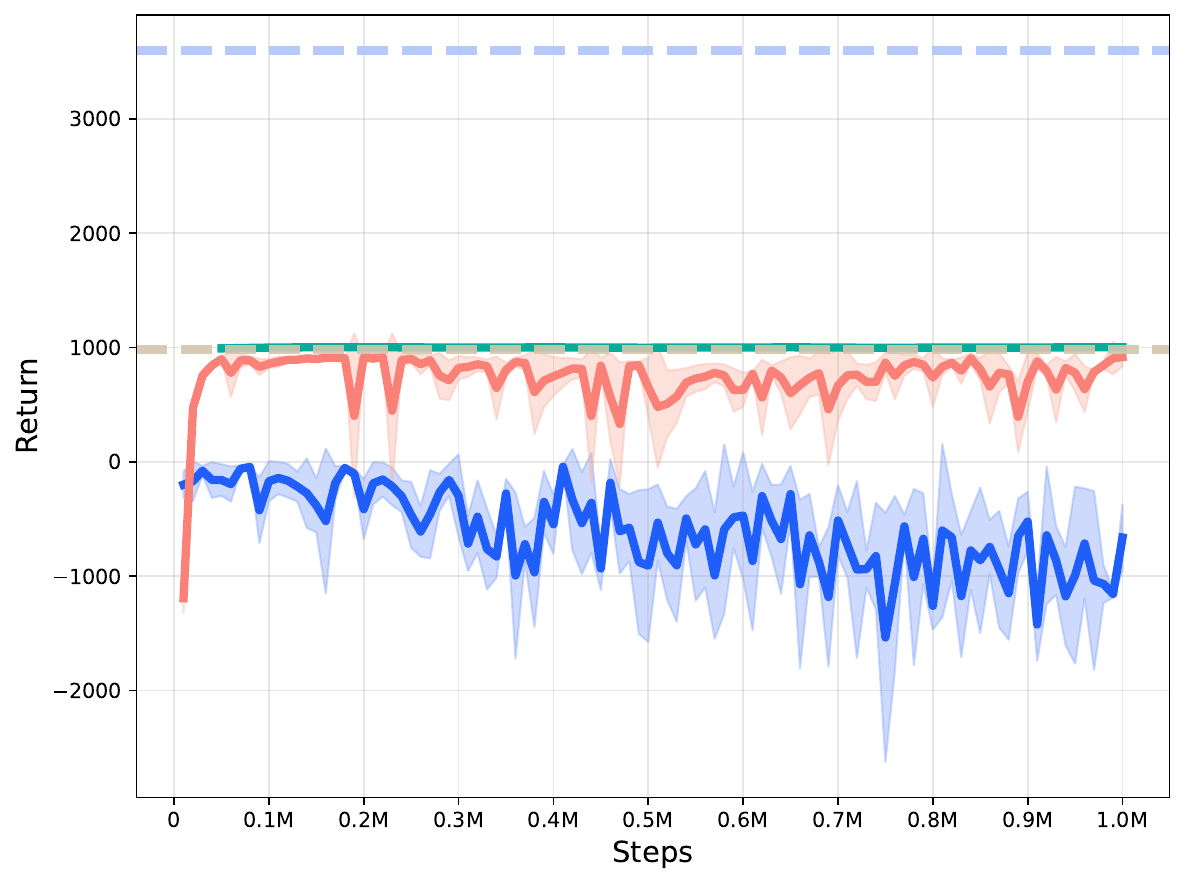}}
        \subfigure[Delay=10, $\#$Traj=1000]{\includegraphics[width=0.33\linewidth]{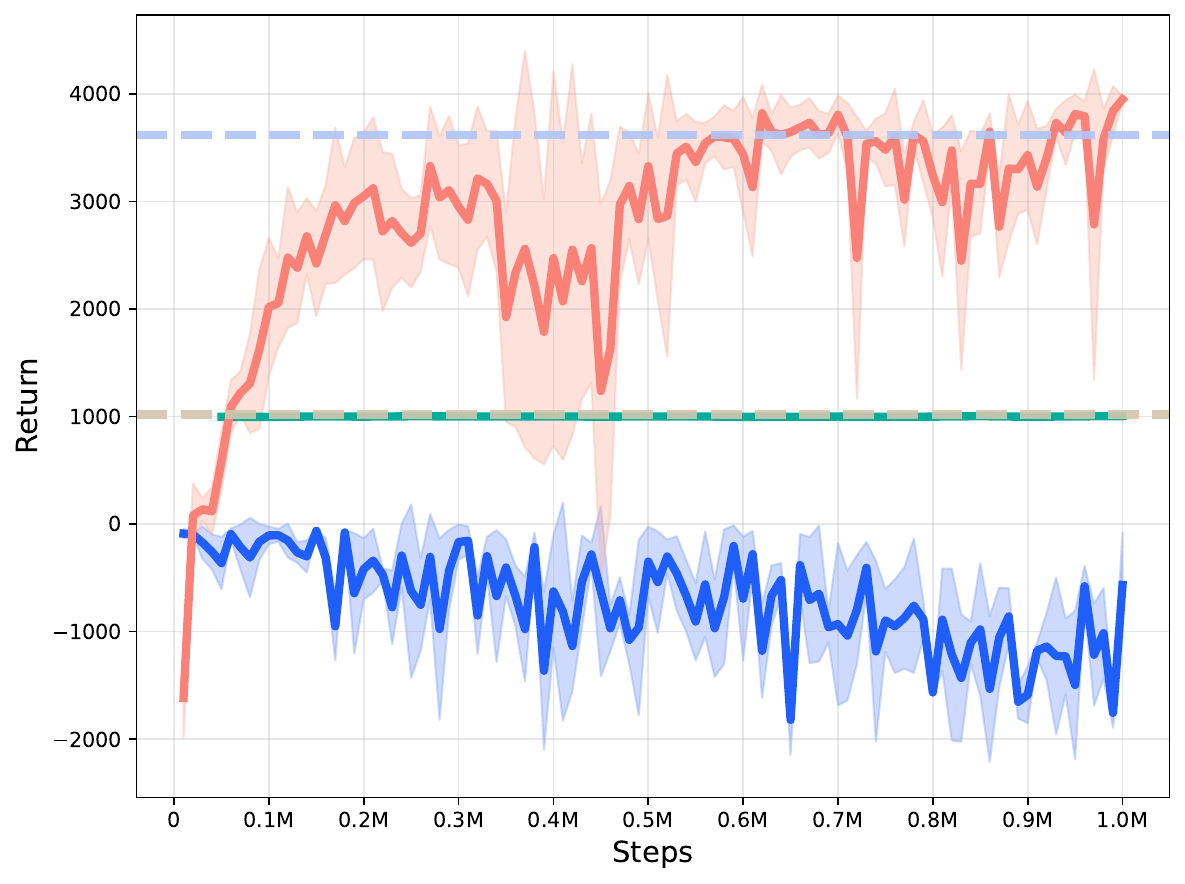}}
    }
    \centerline{
        \subfigure[Delay=25, $\#$Traj=10]{\includegraphics[width=0.33\linewidth]{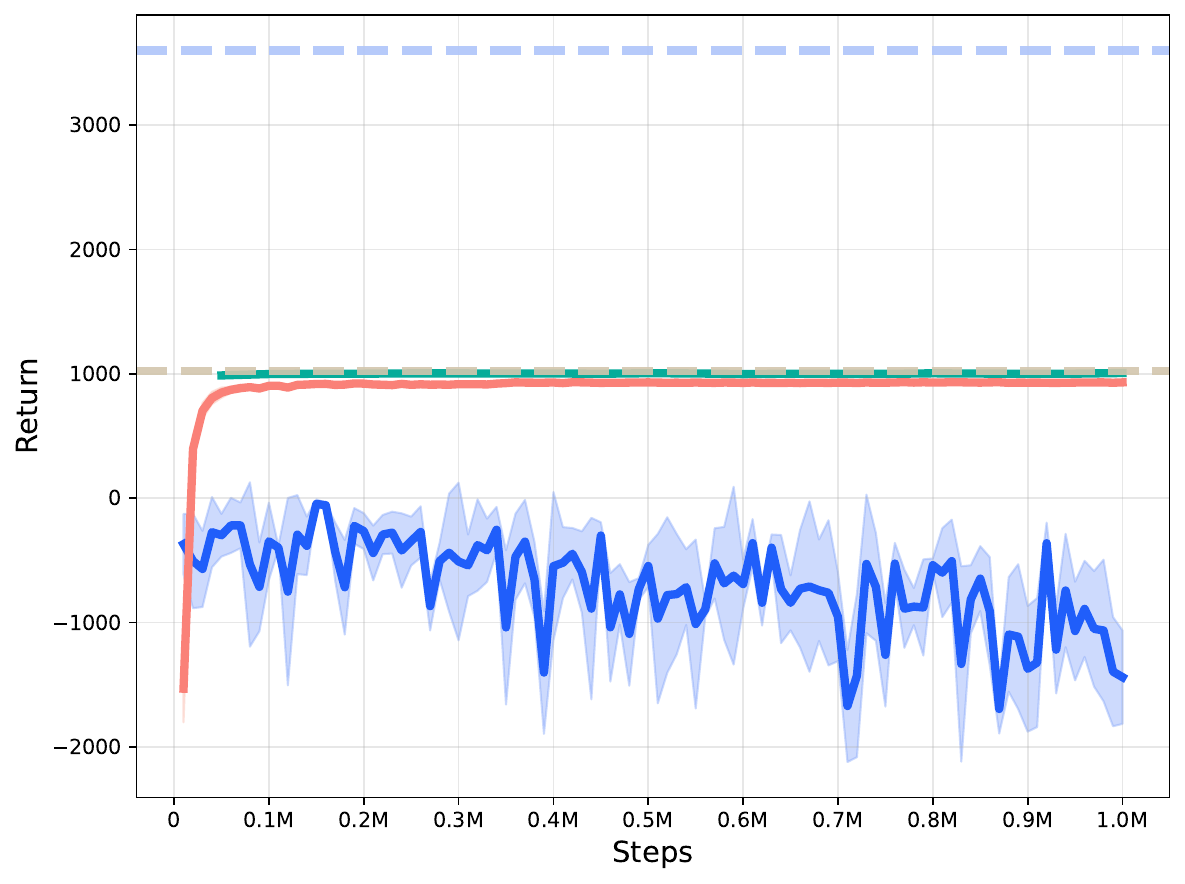}}
        \subfigure[Delay=25, $\#$Traj=100]{\includegraphics[width=0.33\linewidth]{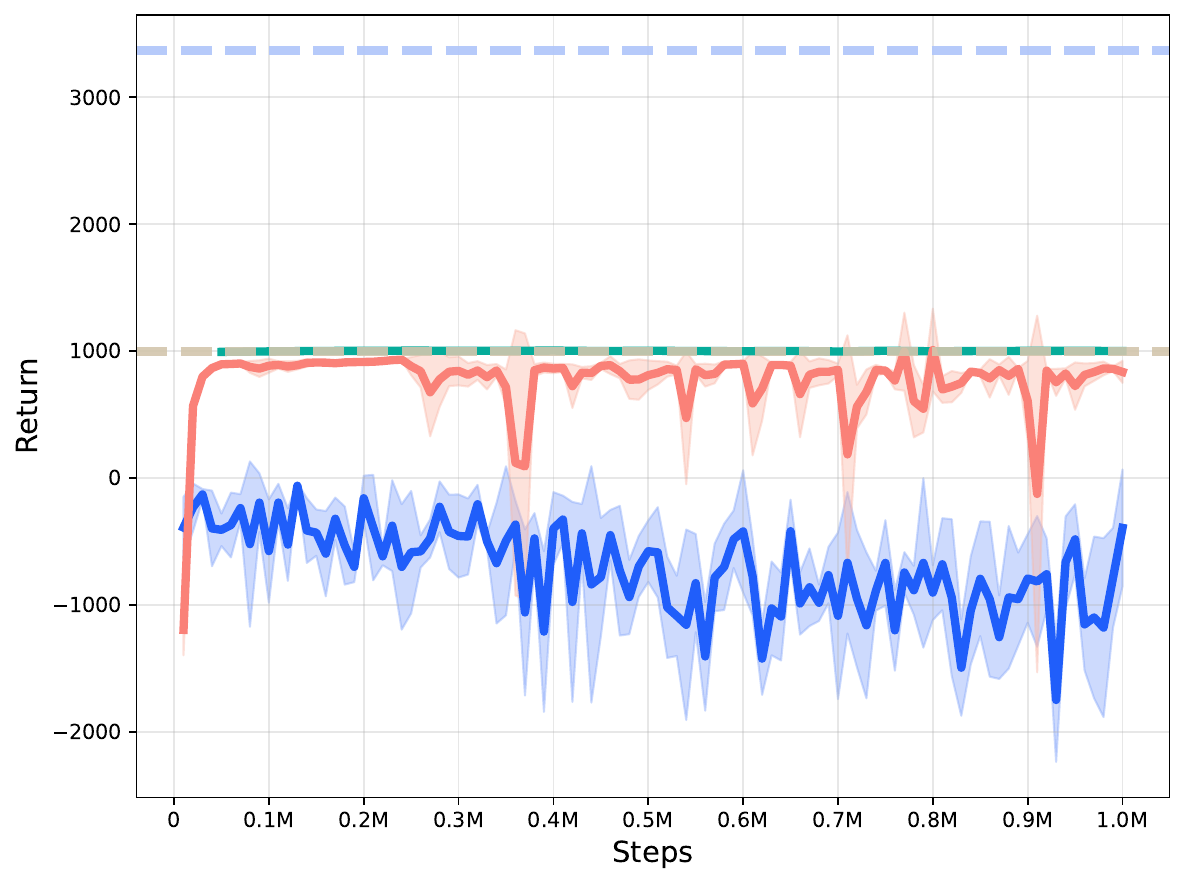}}
        \subfigure[Delay=25, $\#$Traj=1000]{\includegraphics[width=0.33\linewidth]{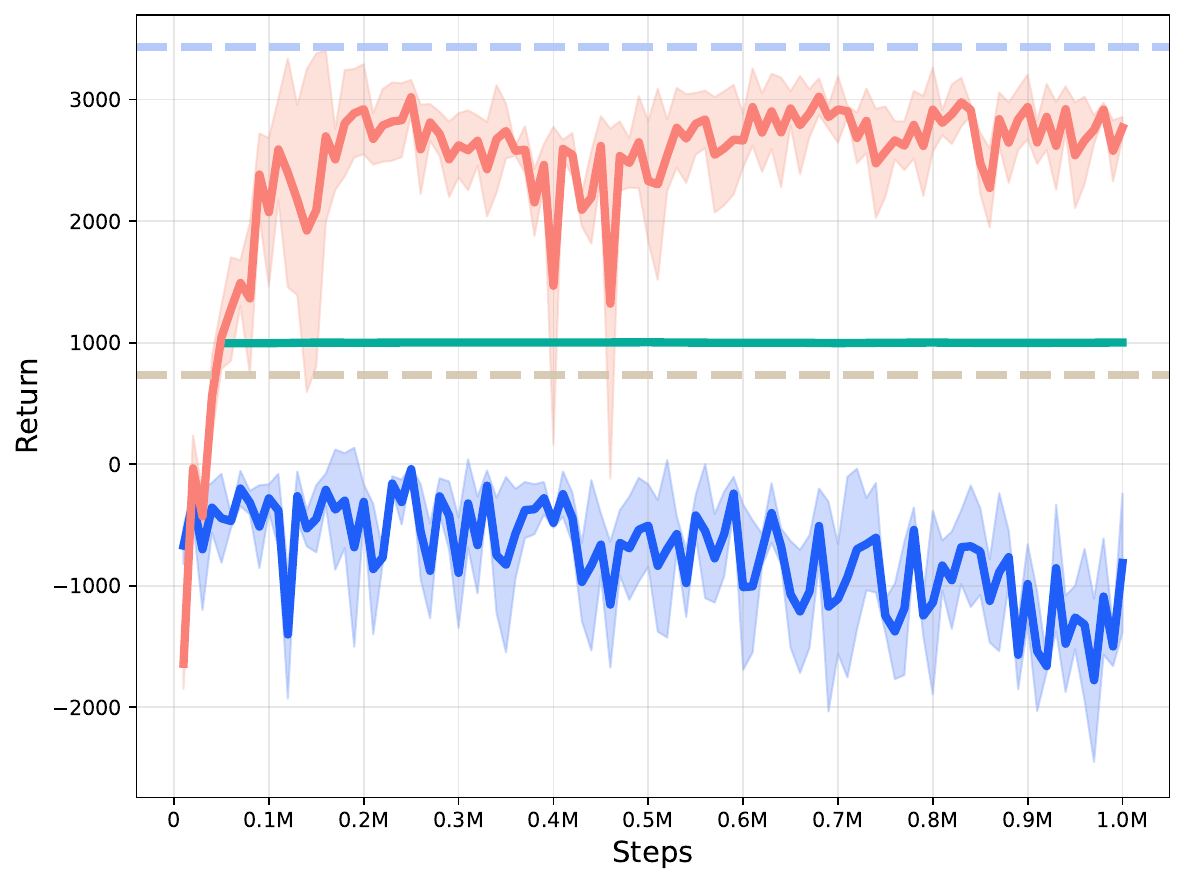}}
    }
    \centerline{
        \includegraphics[width=0.8\linewidth]{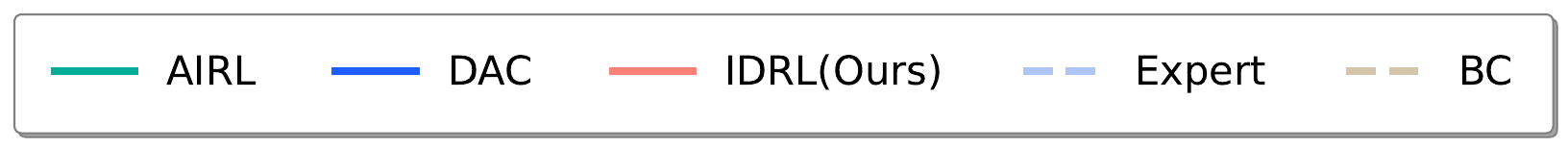}
    }
    \caption{Learning Curves on \texttt{Ant-v4} with different delays and quantities of expert demonstrations.}
\end{figure}

\clearpage
\section{Algorithm}
\label{appendix::algorithm}
\begin{algorithm*}[h]
\caption{Auxiliary Delay Policy Optimization~\citep{wu2024boosting}}
\label{algo:policy_optimization}
\begin{algorithmic}[1]
   \STATE {\bfseries Input:} actor \textcolor{lightblue}{$\pi_\psi$} for delay \textcolor{lightblue}{$\Delta$};
        actor \textcolor{lightred}{$\pi^{\tau}_\phi$}, critics \textcolor{lightred}{$Q^{\tau}_{\theta_1}, Q^{\tau}_{\theta_2}$} and target critics \textcolor{lightred}{$\hat{Q^{\tau}}_{\hat{\theta_1}}, \hat{Q^{\tau}}_{\hat{\theta_2}}$} for auxiliary delays \textcolor{lightred}{$\Delta^{\tau}$}; n-step $n$;
        replay buffer $\mathcal{D}$;
     \FOR{each batch $\{\textcolor{lightblue}{x_t}, \textcolor{lightred}{x^{\tau}_t}, a_t, r_t:r_{t+n-1}, \textcolor{lightblue}{x_{t+n}}, \textcolor{lightred}{x^{\tau}_{t+n}}\} \sim \mathcal{D}$}
            \STATE Compute TD Target $\mathbb{Y} = \sum_{i=0}^{n-1}\left[\gamma^i r_{t+i}\right] + \gamma^n \min\left(\mathbb{Y}_1, \mathbb{Y}_2\right)$, where
            $$
            \mathbb{Y}_1 =
            \mathop{\mathbb{E}}_{\textcolor{lightred}{\hat{a}\sim\pi^{\tau}_\phi(\cdot|x^{\tau}_{t+n})}
            }\left[
                \left(
                \textcolor{lightred}{Q^{\tau}_{\theta_1}(x^{\tau}_{t+n}, \hat{a})}-\log \textcolor{lightred}{\pi^{\tau}_\phi(\hat{a}|x^{\tau}_{t+n})}
                \right)
            \right],
            $$
            $$
            \mathbb{Y}_2 =
            \mathop{\mathbb{E}}_{\textcolor{lightblue}{\hat{a}\sim\pi_\psi(\cdot|x_{t+n})}
            }\left[
                \left(
                \textcolor{lightred}{Q^{\tau}_{\theta_2}(x_{t+n}, \textcolor{lightblue}{\hat{a}})}-\log \textcolor{lightblue}{\pi_\phi(\hat{a}|x_{t+1})}
                \right)
            \right].
            $$
            % \STATE Compute TD Target 
            % $$
            % \mathbb{Y} =
            % \mathop{\mathbb{E}}_{\textcolor{lightred}{\hat{a}\sim\pi^{\tau}_\phi(\cdot|x^{\tau}_{t+n})}\atop
            % \textcolor{lightblue}{\hat{a}\sim\pi_\psi(\cdot|x_{t+n})}
            % }\left[
            %     \sum_{i=0}^{n-1}\left[\gamma^i r_{t+i}\right] + \gamma^n \min \left(
            %     \textcolor{lightred}{Q^{\tau}_{\theta_1}(x^{\tau}_{t+n}, \hat{a})}-\log \textcolor{lightred}{\pi^{\tau}_\phi(\hat{a}|x^{\tau}_{t+n})},
            %     \textcolor{lightred}{Q^{\tau}_{\theta_2}(x_{t+n}, \textcolor{lightblue}{\hat{a}})}-\log \textcolor{lightblue}{\pi_\phi(\hat{a}|x_{t+1})},
            %     \right)
            % \right]
            % $$
            \STATE Update \textcolor{lightred}{$Q^{\tau}_{\theta_i}(i=1,2)$} via applying gradient descent\\
            $$    
            \begin{aligned}
                \triangledown_{\textcolor{lightred}{\theta_i}} \left[\textcolor{lightred}{Q^{\tau}_{\theta_i}(x^{\tau}_t, \textcolor{black}{a_t})}
                -
                \mathbb{Y}
                \right]. \\
            \end{aligned}
            $$
            \IF{$\text{Uniform}(0, 1) > 0.5$}
            \STATE Update \textcolor{lightred}{$\pi^{\tau}_\phi$} via applying gradient descent\\
            $$\triangledown_{\textcolor{lightred}{\phi}} \mathop{\mathbb{E}}_{\textcolor{lightred}{\hat{a}\sim\pi^{\tau}_\phi(\cdot|x^{\tau}_t)}}\left[\log \textcolor{lightred}{\pi^{\tau}_\phi(\hat{a}|x^{\tau}_t)}-\min_{i=1,2}\textcolor{lightred}{Q^{\tau}_{\theta_i}(x^{\tau}_t, \hat{a})}\right].$$
            \ELSE
            \STATE Update \textcolor{lightblue}{$\pi_\psi$} via applying gradient descent\\
            $$\triangledown_{\textcolor{lightblue}{\psi}} \mathop{\mathbb{E}}_{\textcolor{lightblue}{\hat{a}\sim\pi_\psi(\cdot|x_t)}}\left[\log \textcolor{lightblue}{\pi_\psi(\hat{a}|x_t)}-\min_{i=1,2}\textcolor{lightred}{Q^{\tau}_{\theta_i}}(\textcolor{lightred}{x^{\tau}_t}, \textcolor{lightblue}{\hat{a}})\right].$$
            \ENDIF
            \STATE Soft update target critics weights \textcolor{lightred}{$Q^{\tau}_{\theta_1}, Q^{\tau}_{\theta_2}$} via copying from \textcolor{lightred}{$\hat{Q^{\tau}}_{\hat{\theta_1}}, \hat{Q^{\tau}}_{\hat{\theta_2}}$}, respectively.\\
        \ENDFOR
       \STATE {\bfseries Output:} actor \textcolor{lightblue}{$\pi_\psi$};
\end{algorithmic}
\end{algorithm*}
\end{document}